\documentclass[lettersize, journal]{IEEEtran}
\usepackage{times}
\usepackage[noadjust]{cite}
\pdfoutput=1
\usepackage[bookmarks=true, hidelinks]{hyperref}

\usepackage{graphics} %
\usepackage{graphicx}
\usepackage{color}
\usepackage{amsmath} 
\usepackage{amssymb}
\usepackage{amsthm}
\usepackage{mathtools}
\usepackage{verbatim}
\usepackage{textcomp}
\usepackage{soul}
\usepackage{xurl}
\usepackage{hhline} %
\usepackage{xcolor} %
\usepackage{ulem}
\usepackage{dcolumn} %

\usepackage{caption}
\usepackage{subcaption}
\usepackage[ruled, linesnumbered, vlined]{algorithm2e}
\def\BibTeX{{\rm B\kern-.05em{\sc i\kern-.025em b}\kern-.08em
    T\kern-.1667em\lower.7ex\hbox{E}\kern-.125emX}}
\usepackage{balance}

\usepackage{mathtools}

\usepackage{upgreek}
\usepackage{bm}

\renewcommand{\b}[1]{\boldsymbol{#1}} %

\newcommand{\DG}[1]{{\textcolor{black}{{#1}}}}

\newcommand{\ho}[1]{#1}%

\newcommand{\lf}[1]{#1}%
\newcommand{\h}[1]{{{#1}}}
\newcommand{\hh}[1]{{{#1}}}
\newcommand\bluesout{\bgroup\markoverwith{\textcolor{blue}{\rule[0.5ex]{2pt}{0.4pt}}}\ULon}

\newcommand{\systems}{robot }
\newcommand{\system}{robot}
\newcommand{\Systems}{Robot }

\newcommand{\exceeds}[1]{#1}%
\theoremstyle{definition}
\newtheorem{theorem}{Theorem}
\newtheorem{problem}{Problem}
\newcommand{\probref}[1]{Problem~\ref{#1}}

\newtheorem{definition}{Definition}
\newtheorem{lemma}{Lemma}

\newtheorem{assumption}{Assumption}
\theoremstyle{remark}

\begin{document}

\title{Scenario-Based Motion Planning with \\Bounded Probability of Collision}%

\author{Oscar de Groot, Laura Ferranti, Dariu Gavrila, Javier Alonso-Mora
\thanks{The authors are with the Dept. of Cognitive Robotics, TU Delft, 2628 CD Delft, The Netherlands. \texttt {Email: o.m.degroot@tudelft.nl}}
\thanks{This work received support from the Dutch Science Foundation \DG{NWO-TTW}, within the SafeVRU project (14667) and Veni project HARMONIA (18165), and the European Union, within the ERC Starting Grant INTERACT (101041863). Views and opinions expressed are however those of the author(s) only and do not necessarily reflect those of the European Union or the European Research Council Executive Agency. Neither the European Union nor the granting authority can be held responsible for them.}}

\maketitle

\begin{abstract} %

Robots will increasingly operate near humans that introduce uncertainties in the motion planning problem due to their \ho{complex} nature. Typically, chance constraints are introduced in the planner to optimize performance while guaranteeing probabilistic safety. However, existing methods do not consider the actual probability of collision for the planned trajectory, but rather its marginalization, that is, the independent collision probabilities for each planning step and/or dynamic obstacle, resulting in conservative trajectories. To address this issue, we introduce a novel \ho{real-time capable} \ho{method termed Safe Horizon MPC,} that explicitly constrains the joint probability of collision with all obstacles over the duration of the motion plan. This is achieved by reformulating the chance-constrained planning problem using scenario optimization and predictive control. \ho{Our method is less conservative than state-of-the-art approaches, applicable to arbitrary probability distributions of the obstacles' trajectories, computationally tractable and scalable.} We demonstrate our proposed approach using a mobile robot and an autonomous vehicle in an environment shared with humans.
\end{abstract}

\begin{IEEEkeywords}
Motion and Path Planning, Optimization and Optimal Control, Collision Avoidance, Probability and Statistical Methods
\end{IEEEkeywords}

\IEEEpeerreviewmaketitle

\section{Introduction}\label{sec:introduction}
\IEEEPARstart{M}{obile} robots can improve our quality of life, from transporting goods in warehouses~\cite{simon_inside_2019} to helping us commute more efficiently and safely using self-driving vehicles~\cite{walker_self-driving_2019}. In most applications where robots are currently deployed, the operating domain does not allow the robots to operate near humans (e.g., using dedicated lanes) to simplify the robot-navigation task. To deploy mobile robots in real-world environments, such as our cities, they need to be capable of navigating among humans, which remains challenging. 

In crowded environments, the robot needs to understand and infer the motion of humans in order to move safely and efficiently. Unfortunately, human behavior is hard to predict and varies per person. In addition, human intentions cannot be explicitly communicated to the \system. This inherently makes the motion prediction of humans non deterministic. Recent perception methods, \lf{such as}\ho{~\cite{kooij_context-based_2019}, infer human intentions,} returning probabilistic information concerning humans. These probabilistic predictions need to be considered and exploited by the motion planner.\looseness=-1

In a probabilistic setting, the occurrence of a collision is a \textit{probabilistic} event. Our goal is to assess and bound the \textit{Collision Probability} (CP) of the planned trajectory. This is challenging because of the non-Gaussian and multi-modal, i.e., multiple paths are possible, distributions involved when predicting human behavior. In addition, when the collision probability spans a duration, e.g., over the planned trajectory, then it must consider the correlation over time. The time correlation exists because the first collision in a trajectory renders it unsafe, such that all subsequent collisions can be ignored. Almost all of previous works\ho{, e.g., \cite{zhu_chance-constrained_2019, luders_chance_2010, blackmore_probabilistic_2006, van_den_berg_lqg-mp_2011, wang_non-gaussian_2020, wang_fast_2020, de_groot_scenario-based_2021},} do not account for this correlation,  which leads to conservative trajectories. Similarly, the collision probability \ho{in these works} is computed per obstacle, which degrades performance when more obstacles influence the plan, i.e., in crowded environments. By accounting for the correlation in time and by considering all obstacles, one would be able to plan more efficient trajectories, without compromising safety.

\begin{figure}[t]
    \centering
    \includegraphics[width=0.45\textwidth]{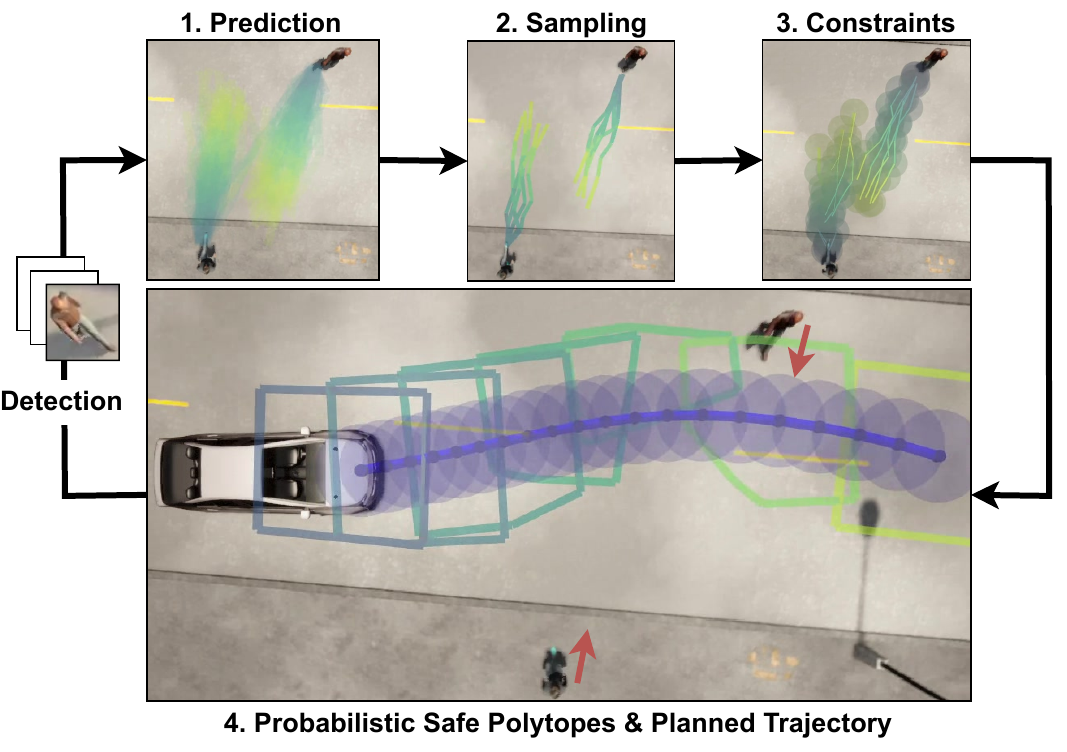}
    \caption{Overview of the proposed scenario-based motion planner. \ho{Predicted obstacle trajectories are sampled to obtain scenarios, where each scenario represents a trajectory for all obstacles over the planning horizon. By ensuring safety for all scenarios, probabilistic safety of the motion plan is guaranteed.}}
    \label{fig:eye_catcher}
\end{figure}

We achieve this goal by devising a novel probabilistic safe motion planner based on scenario optimization and nonlinear Model Predictive Control (MPC). Nonconvex Scenario optimization~\cite{campi_general_2018} is a sampling based approach for assessing the probability that a constraint holds under uncertainty. In the context of this paper, the imposed constraint \ho{probabilistically} enforces collision avoidance between the controlled \systems and the dynamic obstacles over the duration of the motion plan (see Fig.~\ref{fig:eye_catcher}). The uncertainty associated with the motion of detected obstacles is predicted forward in time (Step 1). We sample scenarios from these predictions that describe the trajectories of all dynamic obstacles during the planning horizon (Step 2) and construct collision avoidance constraints around each of the scenarios (Step 3). The \systems trajectory is optimized with respect to the constraints (Step 4) which provides guaranteed probabilistic collision avoidance. \ho{The exact guarantees are tied to the number of sampled scenarios, which we determine before deploying the controller.} This is, to our knowledge, the first real-time capable method that directly provides probabilistic safety guarantees on the planned trajectory, in contrast with the existing state-of-the-art where the same quantity is conservatively approximated through its marginals. We refer to our novel probabilistic safe motion planner as \textit{Safe Horizon MPC} (SH-MPC).

We compare our method, for which we provide theoretical guarantees, to several baselines, in multiple scenarios, both with mobile robots and self-driving vehicles, and for varied probability distributions. We show that our planner is less conservative and improves efficiency, in terms of traveling time, while remaining safe.

\section{Related Work and Contribution}
A sizable body of literature exists on motion planning for safe autonomous navigation. In this section, we review some of the existing work on trajectory optimization with collision avoidance under uncertainty. For a general overview of autonomous navigation, we refer the reader to~\cite{paden_survey_2016} and~\cite{schwarting_planning_2018}. 

\subsection{Collision Avoidance under Uncertainty}
An important problem in autonomous navigation is to prevent collisions with \ho{dynamic obstacles}. In trajectory optimization, the navigation problem is formulated as an optimization problem where performance criteria are optimized (e.g., lane following and progress towards the goal) under constraints (e.g., collision avoidance and actuator limits)~\cite{schwarting_safe_2018},~\cite{brito_model_2019}. Due to large and multi-modal uncertainties in human motion prediction, collision avoidance in the mean or nominal case (e.g., constant velocity) may lead to collisions in practice. Many works therefore consider how to address this uncertainty. %

One can consider the collision avoidance problem as a special case of optimization under uncertainty, for which two common approaches exist. \textit{Robust optimization}~\cite{ben-tal_robust_1998} requires the constraints to be satisfied for all possible realizations of the uncertainty, while \textit{stochastic optimization}~(see \cite{mesbah_stochastic_2016} for an overview) allows for violation of the constraints, as long this happens with a probability smaller than an upper bound $\epsilon$. \ho{Because the set of all possible realizations is often not available or too conservative, we focus in the remainder of this section on stochastic optimization.} 
We refer to~\cite{kouvaritakis_model_2016} for more details on both methods in the context of Model Predictive Control (MPC). 

\subsection{Collision Avoidance Chance Constraints}
The probability of constraint violation in stochastic optimization is specified through \textit{chance constraints}, which constrain the \textit{probability} that a nominal constraint is satisfied. Exact evaluation of these chance constraints is however intractable in almost all applications. Existing works therefore focus on an approximation of the constraints, through additional assumptions on the probability distribution associated with the uncertainty (e.g., Gaussian~\cite{zhu_chance-constrained_2019}, \cite{fisac_probabilistically_2018}) or on the controlled \systems (e.g., linear dynamics~\cite{blackmore_probabilistic_2006}). Recent works~\cite{wang_fast_2020},~\cite{de_groot_scenario-based_2021} have resolved many of the assumptions, making the framework applicable to nonlinear \systems dynamics and arbitrary probability distributions. However, the chance constraints in these and many other works are not imposed on the {\system}'s planned trajectory, i.e., the timed sequence of planned positions, but rather on each of its individual positions over the planning horizon. \ho{This fails to accurately bound the probability of colliding at \textit{any} time\footnote{\ho{Similarly, chance constraints imposed per obstacle fail to estimate the probability of colliding with \textit{any} obstacle.}}.} In this regard, three types of chance constraint formulations have been considered in previous work: \textit{Marginal}, \textit{Conditional} and \textit{Joint}, summarized in Table~\ref{tab:related_work}.%

\begin{table}[t] %
    \centering
    \begin{tabular}{|l|c|c|c|}
    \hline \multicolumn{1}{|c|}{\textbf{Method}}&\begin{tabular}{@{}c@{}}\textbf{Non} \\ \textbf{Gaussian}\end{tabular} & \textbf{Dynamics} & \begin{tabular}{@{}c@{}}\textbf{Computation} \\ \textbf{Times\textsuperscript{*}}\end{tabular} \\\hline
    \begin{tabular}{@{}l@{}}Additive Marginal \\ \cite{zhu_chance-constrained_2019, luders_chance_2010, blackmore_probabilistic_2006, masahiro_ono_iterative_2008}\end{tabular} & & Linear & $10$~ms\\\hline
    \begin{tabular}{@{}l@{}}Multiplicative Marginal\\ \cite{van_den_berg_lqg-mp_2011} \end{tabular} & & Linear & $100$~ms\\\hline
    \begin{tabular}{@{}l@{}}Additive Marginal \\ \cite{wang_non-gaussian_2020, wang_fast_2020, de_groot_scenario-based_2021} \end{tabular} & X & Nonlinear & $10$~ms\\\hline\hline
    \begin{tabular}{@{}l@{}}Conditional Marginal \\ \cite{patil_estimating_2012} \end{tabular} & & Linear & $10$~ms\\\hline\hline
    \begin{tabular}{@{}l@{}}Joint Trajectory\\ \cite{blackmore_probabilistic_2010, janson_monte_2015}\end{tabular} & X & Linear & $1000$~ms\\\hline
    \begin{tabular}{@{}l@{}}Joint Trajectory\\ \cite{schmerling_evaluating_2017} (Estimation only)\end{tabular} & X & Nonlinear & $1000$~ms\\\hline
    \begin{tabular}{@{}l@{}}Joint Trajectory\\ \textbf{Ours} \end{tabular} & \textbf{X} & \textbf{Nonlinear} & \textbf{$\b{10}$~ms}\\\hline
    \end{tabular}
    \caption{Comparison of motion planning method with collision avoidance chance constraints. \textsuperscript{*}Order of magnitude, exact times vary with the acceptable CP, application and computational hardware.}
    \label{tab:related_work}
\end{table}

\subsection*{\ho{Marginal Chance Constraints}}
Constraints on each position along the trajectory are referred to as \textit{marginal} chance constraints. 
\h{Let event $A_k$ denote the case that no collisions occur at time $k$ and {\small$\mathbb{P}(A_k)$} therefore be the probability that the robot is safe at timestep $k$. Then} the exact probability that a trajectory is safe \h{over $N$ steps} is given by {\small$\mathbb{P}(A) = \h{\prod_{k=1}^N} \mathbb{P}(A_k \ | \ A_{0:k-1})$}. That is, the CP for each position is conditional on the probability of avoiding collisions up until the position is reached at time $k$. The problem is simplified if we assume instead that this event is independent for all states ({\small$\Tilde{\mathbb{P}}(A) \approx \prod_k \mathbb{P}(A_k)$}). In~\cite{janson_monte_2015}, these marginal methods are further divided into \textit{additive} and \textit{multiplicative} approaches. 

\textit{Additive} approaches impose constraints on each marginal probability (i.e., $\mathbb{P}(A_k)\leq \epsilon_k$). Using Boole's inequality {(\small$\mathbb{P}(\cup_k A_k) \leq \sum_k\mathbb{P}(A_k)$)} the CP of the trajectory is bounded by the sum of the individual CPs. 
Under Gaussian uncertainty, the work~\cite{blackmore_probabilistic_2006} reformulated the constraints as an analytical constraint on the 1D Cumulative Density Function (CDF). The same idea is used in~\cite{zhu_chance-constrained_2019} in an MPC framework to prevent collision between \systems and obstacle volumes. In~\cite{masahiro_ono_iterative_2008} the bound on each marginal probability is updated, known as \textit{risk allocation}, while maintaining the same total risk bound (i.e., $\sum_k \epsilon_k = \epsilon$). 
In~\cite{luders_chance_2010, aoude_probabilistically_2013} marginal chance constraints are applied to the Rapidly expanding Random Trees (RRT) algorithm such that each node in the tree is statistically safe. When the uncertainty is non Gaussian, the CDF of the probability distribution is typically not available. In~\cite{wang_non-gaussian_2020, wang_fast_2020} an MPC for motion planning is formulated where inequalities are posed on stochastic moments of the marginal probability distribution. \h{A similar approach, for linear dynamics, is applied in~\cite{ren_chance-constrained_2023} where the conditional Variance-at-Risk (cVaR) is used to minimize constraint violation.} In our previous work~\cite{de_groot_scenario-based_2021}, we ensure safety for a finite set of sampled chance constraints and generalize the associated safety properties using the scenario approach~\cite{campi_general_2018}. 

The \textit{multiplicative} formulation explicitly constrains the product of the marginal probabilities. It was applied in~\cite{van_den_berg_lqg-mp_2011} to plan motion under sensing and actuation uncertainty. An alternative marginal formulation is proposed in~\cite{fisac_probabilistically_2018}, which bounds the largest marginal constraint violation.

The limitation of marginal approaches is that the bound on the CP of the trajectory is inaccurate, as noted by~\cite{patil_estimating_2012}  and~\cite{janson_monte_2015}. It is shown in~\cite{janson_monte_2015} that the trajectory CP approaches $\infty$ and $1$, for the additive and multiplicative formulation, respectively, when the number of evaluations in the trajectory increases, regardless of the real CP. \h{Marginal constraints only asses the risk correctly for the first time step and a single obstacle. The risk of the remainder is under- or overestimated.} \hh{Overestimation of the risk and the associated unsafe space along the time horizon can cause the planning problem to become infeasible and may cause the robot to freeze.} \ho{Due to these limitations,~\cite{patil_estimating_2012} conditioned marginal chance constraints on being collision free at prior times and evaluated them by truncating the part of the distribution in collision in each time instance. This formulation is more accurate, but is limited to Gaussian distributions.}

\subsection*{Joint Chance Constraint}
Some authors formulate a \textit{joint} chance constraint on the CP of the planned trajectory. \hh{Joint chance constraints estimate the open-loop risk over the time horizon more accurately, making it less likely that the problem becomes infeasible and improving performance.}
\h{The joint CP} can be evaluated by using sampling-based methods~\cite{janson_monte_2015}. In particular, prior works~\cite{blackmore_probabilistic_2010, janson_monte_2015, schmerling_evaluating_2017} consider Importance Sampling Monte-Carlo (ISMC) sampling to \textit{approximate} the CP. An empirical estimate of the constraint violation is obtained by sampling the joint distribution and evaluating the constraint for each sampled \textit{trajectory}. %
ISMC methods \h{are well suited for estimating risk, but are computationally expensive when planning a trajectory. An alternative is to formulate a mixed-integer problem (e.g.,~\cite{blackmore_probabilistic_2010}) to decide which samples may be violated, but these problems are hard to solve in real-time.} %

\subsection{Scenario Optimization}
Rather than constraining each of the marginal CPs along the planned trajectory we directly constrain the joint CP of the trajectory. We achieve this through Nonconvex Scenario Optimization (NSO)~\cite{campi_general_2018} under an explicit chance constraint on the joint CP. Scenario optimization is well established for convex optimization under uncertainty~(see e.g.,~\cite{calafiore_scenario_2006, campi_exact_2008, campi_sampling-and-discarding_2011, schildbach_randomized_2013}) and has recently been extended to the general nonconvex case~\cite{campi_general_2018}. NSO is a sampling-based framework for optimization under uncertainty, similar to ISMC, but instead of \textit{averaging} the samples, it \textit{constrains} the solution to the samples, \h{which is computationally more efficient when planning trajectories.}%

The NSO framework~\cite{campi_general_2018} cannot directly be applied to the considered application. \hh{The traditional safety assessment needs to repeat the optimization problem once for each scenario and thus cannot provide guarantees in real-time. In addition, unbounded distributions (e.g., Gaussian distributions) cannot be readily incorporated.} We address these limitations in this work, \hh{significantly reducing the computational complexity of the online safety assessment (Sec.~\ref{sec:algorithms}) and including distributions with unbounded support (Sec.~\ref{sec:removal}).} This allows us to apply the framework to online motion planning in real-time.

\subsection{Contribution}
The contribution of this paper is three-fold:
\begin{enumerate} %
    \item A novel trajectory optimization method, \textit{Safe Horizon MPC}, that explicitly constrains the collision probability over the full duration of the planned trajectory. This distinguishes our work from previous work, e.g., \cite{zhu_chance-constrained_2019, luders_chance_2010, blackmore_probabilistic_2006, van_den_berg_lqg-mp_2011, wang_non-gaussian_2020, wang_fast_2020, de_groot_scenario-based_2021}, where the collision probability is constrained per planning time instance. The idea is to formulate a single collision avoidance constraint over the horizon and to evaluate it by sampling from the distribution of the uncertainty\h{, where each sample represents a trajectory for all obstacles.} The more \h{samples are drawn}, the higher the probability that the constraint is satisfied. By relying on sampling, our planner is distribution agnostic.

    \item An approach that, under a convexity assumption on the iterations of the underlying optimization algorithm, \h{identifies the scenarios that hold the solution in place (known as the \textit{support}) during optimization, in contrast with the framework~\cite{campi_general_2018} where the support is computed after optimization. \hh{We leverage this information to certify the motion plan online.}}

    \item A method to incorporate \textit{scenario removal} in the NSO framework of~\cite{campi_general_2018}. Scenario removal is an extension of the scenario approach where more scenarios are sampled initially such that some of the most limiting scenarios can be removed before optimization. This approach, which has been studied in the convex case (see for example~\cite{campi_sampling-and-discarding_2011}, \cite{garatti_risk_2019}), reduces how conservative the solution of the scenario program is \hh{and makes the scenario approach applicable to distributions with unbounded support.}  
    As a result, scenario removal makes the online optimization consistent and reduces infeasibility.
    
\end{enumerate}
We compare our approach in two simulation environments, on a mobile robot and a self-driving vehicle, with pedestrians against marginal baselines and show that our method achieves faster trajectories while maintaining the probabilistic safety specification. In addition, we show that the computation times scale well with regards to the number of obstacles and the type of distribution. \ho{Our planner is implemented in C++/ROS and will be released open source.}

\subsection{Notation}%
Vectors and matrices are expressed in bold and capital bold notation, respectively. We use $\bigwedge$ to denote the ``and'' operation and $\bigvee$ to denote the ``or'' operation. The notation $\hh{\mathcal{C}} \ \textrm{\textbackslash} \ C_i$ refers to the set $\hh{\mathcal{C}}$ from which the element $C_i$ is removed. %

\section{PROBLEM FORMULATION}\label{sec:problem_formulation}
We consider a controlled \systems with the general nonlinear discrete-time dynamics
\begin{equation*}
 \h{\b{x}_{k + 1} = f(\b{x}_k, \b{u}_k),\label{eq:dynamics}}
\end{equation*}
where $\b{x}_k \in \mathbb{R}^{n_x}$ and $\b{u}_k \in \mathbb{R}^{n_u}$ denote the states and inputs, respectively. \h{The robot state is assumed to contain its $x$-$y$ position $p =[x,y] \in \mathbb{R}^2 \subseteq \mathbb{R}^{n_x}$.} Humans in the environment of the \systems pose constraints on the navigation envelope. 
\h{We consider that the motion prediction of the dynamic obstacles is uncertain by modeling the positions of at most $M$ obstacles as random variables. In particular, we denote the uncertain position of Obstacle $j$ at time step $k$ as $\b{\delta}_{k, j}\subseteq \b{\delta}_k$, where $\b{\delta}_k$ contains all obstacle positions at time $k$. The \textit{joint uncertainty} $\b{\delta} = \left[\b{\delta}_1^T, \hdots, \b{\delta}_N^T\right]^T \in \Delta$ stacks the uncertainty over all time steps.} 
\h{Here, $\Delta$ denotes the probability space\footnote{See e.g., \cite{billingsley_probability_1995-1} for more details.} of the joint uncertainty, which is endowed with a $\sigma$-algebra $\mathcal{D}$ and a probability measure $\mathbb{P}$.}

\subsection{Chance Constrained Planning Problem}
\h{The problem setup is visualized in Fig.~\ref{fig:overall-brief}.} We model the \systems area and obstacle area with $n_d$ discs and a single disc, respectively. \h{The dynamic uncertainty of other road-users affects the navigation envelope of the robot.} To constrain the probability of a collision with \textit{any} obstacle along the horizon, we formulate a single chance constraint for collision avoidance. This leads to the following Chance Constrained Problem (CCP)

\begin{problem}[CCP]\label{prob:ccp}
\begin{subequations}
\label{eq:motion_planning_ccp}
\begin{align}
    \min_{\b{u} \in \mathbb{U}, \b{x}\in\mathbb{X}} \ & \sum_{k = 0}^N J(\b{x}_k, \b{u}_k)\\
    \textrm{s.t.} \quad & \b{x}_0 = \b{x}_{\textrm{init}}\\
    & \b{x}_{k + 1} = f(\b{x}_k, \b{u}_k), \ k = 0, \hdots, N-1\\
    & \mathbb{P}\!\left[\bigwedge_{k = 1}^N\bigwedge_{d = 0}^{n_d}\bigwedge_{j = 0}^M\! \left(||\b{p}^d_k - \bm{\delta}_{k, j}||_{\h{2}} \geq r \right)\right] \geq 1 - \epsilon%
    \label{eq:motion_planning_constraints}
\end{align}
\end{subequations}
\end{problem}
where $\b{x}_{\textrm{init}}$ denotes the initial state of the {\system}, objective $J$ is designed to achieve control objectives, \h{$\mathbb{X} \subseteq \mathbb{R}^{N n_x}$ and $\mathbb{U} \subseteq \mathbb{R}^{N n_u}$ denote state and input constraints over the horizon}, respectively, and \eqref{eq:motion_planning_constraints} is the collision avoidance chance constraint. Our goal is to compute a control input $\b{u}$ and trajectory $\b{x}$ under the uncertainty $\bm{\delta}$ that is collision free with a probability of at least $1-\epsilon$, where $\epsilon$ denotes the maximum collision probability over the planned trajectory. 

\h{When $\mathbb{P}$ is estimated by a prediction model, the collision avoidance constraint is formulated with respect to an estimate $\hat{\mathbb{P}}$ of $\mathbb{P}$ and guarantees relate to the estimate of the probability distribution.}

\subsection{Scenario-Based Planning Problem}
Directly evaluating chance constraint \eqref{eq:motion_planning_constraints} is not computationally feasible in closed loop. Our goal is to formulate a sampled deterministic version of the CCP, known as a Scenario Program (SP)~\cite{campi_general_2018}. Because \eqref{eq:motion_planning_constraints} can be enforced when $\bm{\delta}$ is deterministic, we are able to reformulate the CCP into a problem that can be solved in closed loop. %
\h{The challenges for safe robot navigation within this framework are to determine the number of samples that must be drawn and, consequently, identify the samples that affect the optimization.}

\begin{figure}[t]
    \centering
    \includegraphics[angle=-90, width=0.3\textwidth]{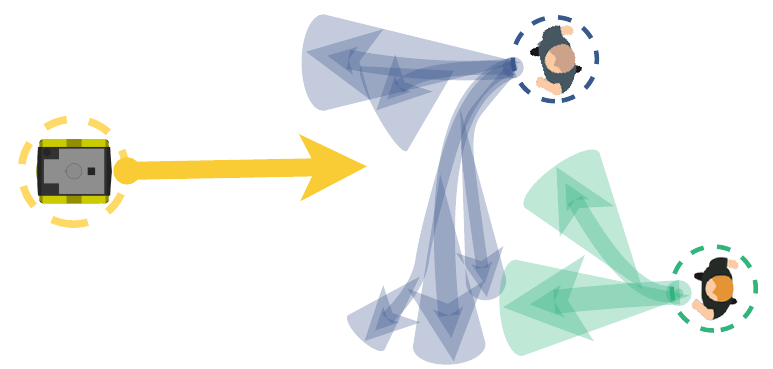}
    \caption{\h{The \h{chance constrained} motion planning problem considered in this paper with the robot (yellow) navigating under probabilistic motion predictions of obstacles (blue/green). \h{The distribution of motion prediction can take any form, but is visualized here with several modes (arrows/shaded regions).}}}
    \label{fig:overall-brief}
\end{figure}

\subsection{Paper Organization}

In the following, we first consider a more general CCP formulation that can be solved by the proposed framework. We provide a brief summary of the nonconvex scenario optimization framework~\cite{campi_general_2018} to show how this general class of CCPs can be solved via its associated SP. We then present our main results in Section~\ref{sec:algorithms} and Section~\ref{sec:removal} which allows the SP to be solved in closed loop, while maintaining feasibility with respect to the CCP. \h{Section~\ref{sec:shmpc} shows how to solve \probref{prob:ccp} by solving its associated SP in an MPC framework. Finally, Sections~\ref{sec:motion_planning} and \ref{sec:results} apply this MPC framework to generate safe motion plans for a \systems navigating among pedestrians.}%

\section{NONCONVEX SCENARIO OPTIMIZATION}\label{sec:preliminaries}
\h{In the following, we summarize the main results of the NSO framework~\cite{campi_general_2018} \h{that we use to build our motion planning framework}. To this end, consider the following generalization of Problem~\ref{prob:ccp}:}
\begin{problem}[General CCP]\label{prob:general_ccp}
\begin{subequations}
\label{eq:general_problem}\label{eq:chance_constrained_problem}
\begin{align}
    \min_{\b{u} \in \mathbb{U}, \b{x}\in\mathbb{X}} \quad & \sum_{k = 0}^N J(\b{x}_k, \b{u}_k)\\
  \textrm{s.t.} \quad & \b{x}_0 = \b{x}_{\textrm{init}}\\
  & \b{x}_{k + 1} = f(\b{x}_k, \b{u}_k),\ k = 0, \hdots, N-1 \label{eq:opt_dynamics}\\
  &\mathbb{P}\left[g(\b{x}, \bm{\delta}) \leq 0\right] \geq 1 - \epsilon, \ \bm{\delta} \in \Delta.\label{eq:opt_constraints}
\end{align}
\end{subequations}
\end{problem}
\h{The} constraints $g(\b{x}, \bm{\delta})\leq 0$ must be satisfied with a probability of at least $1 - \epsilon$. The main idea of scenario optimization~\cite{campi_general_2018} is to solve \probref{prob:general_ccp} by imposing deterministic constraints for a set of \textit{scenarios} $\bm{\omega} = \{\bm{\delta}^{(1)}, \hdots, \bm{\delta}^{(S)}\} \ \h{\in \Delta^S}$, where each scenario is independently extracted from $\mathbb{P}$\h{ and $\Delta^S$ represents the S-fold Cartesian product of $\Delta$ associated with drawing $S$ random samples. %
The number of sampled scenarios is known as the \textit{sample size} $S$. The SP associated with \probref{prob:general_ccp} is given by}
\begin{problem}[General SP]\label{prob:general_sp}
\begin{subequations}
\label{eq:scenario_program}
\begin{align}
    \min_{\b{u} \in \mathbb{U}, \b{x}\in\mathbb{X}} \quad & \sum_{k = 0}^N J(\b{x}_k, \b{u}_k)\\
    \textrm{s.t.} \quad & \b{x}_0 = \b{x}_{\textrm{init}}\\
    & \b{x}_{k + 1} = f(\b{x}_k, \b{u}_k), \ k = 0, \hdots, N-1\label{eq:scenario_eq_constraints}\\
    &g(\b{x}, \bm{\delta}^{(i)}) \leq 0, \ \bm{\delta}^{(i)} \in \h{\b{\omega}}, \ i = 1 , \hdots, S,\label{eq:scenario_constraints}
\end{align}
\end{subequations}
\end{problem}
where chance constraint~\eqref{eq:opt_constraints} has been replaced by the deterministic constraints for each of the scenarios in~\eqref{eq:scenario_constraints}.

\h{To simplify notation, w}e define a \textit{decision} as $\bm{\theta} \coloneqq \begin{bmatrix}\b{x}^T, \b{u}^T\end{bmatrix}^T\in \Theta$, where $\Theta = \mathbb{X} \times \mathbb{U}$. Each scenario $\bm{\delta}^{(i)}$ \h{imposes a constraint} $\bm{\theta} \in \Theta_{\bm{\delta}^{(i)}} \h{\subset \Theta}$ \h{on the decision}. Formally, to make a decision based on the scenarios we have a \textit{decision algorithm} $\mathcal{A}$, that maps the scenarios to a decision (i.e., solving \probref{prob:general_sp}). This \h{(sub)optimal} decision $\b{\theta}^* = \mathcal{A}(\h{\b{\omega}})$ is called the \textit{scenario decision}. \hh{The following assumption needs to hold.}

\begin{assumption}\label{as:decision_algorithm}%
\hh{For any finite $S = 1, 2, \hdots$ and for any sample $\h{\b{\omega}}\in \Delta^S$, it holds that}
\begin{equation}
    \h{\mathcal{A}(\b{\omega}) \in \Theta_{\b{\delta}^{(i)}}, \ i = 1, \hdots, S.}
\end{equation}
\end{assumption}
\h{We impose a finite sample size, contrary to~\cite{campi_general_2018}, to focus on the considered application in which infinite sample sizes are not relevant.} \hh{Assumption~\ref{as:decision_algorithm} implies that the decision algorithm identifies a feasible solution with respect to the scenario constraints for all possible extractions of samples. We show in Sec.~\ref{sec:removal} how this assumption can be satisfied under unbounded uncertainty.} %
Intuitively, the more scenarios that were used to compute the scenario decision, the lower is the probability that the resulting decision will violate the constraints.
Formally, the \textit{violation probability}, $V : \Theta \rightarrow [0, 1]$, given by
\begin{equation}
    V(\bm{\theta}) \coloneqq \mathbb{P}\left[ \bm{\delta} \in \Delta : \bm{\theta} \notin \Theta_{\bm{\delta}} \right],\label{eq:violation_probability}
\end{equation}
defines the probability that a decision $\bm{\theta}$ {violates} a newly observed scenario. We also refer to this probability as the \textit{risk} of the decision. Since the decision $\bm{\theta}$ depends on the realization of the randomly sampled scenarios, the violation probability $V(\bm{\theta})$ is in itself a random variable over the product probability measure, given by $\mathbb{P}^{\textrm{S}}$ = $\mathbb{P} \times \hdots \times \mathbb{P}$ (S times). We are therefore interested in lower bounding the \textit{confidence}, which is the \textit{probability} that the scenario decision achieves a risk of at most $\epsilon$. The key variable to obtain this bound is the support subsample, defined as follows.

\begin{definition}\label{def:support}
\cite{campi_general_2018}: Given a multi-sample $\h{\b{\omega}} = \h{\{}\bm{\delta}^{(1)}, \hdots, \bm{\delta}^{(S)}\h{\}}$, a \textit{support subsample} $\hh{\mathcal{C}} = \h{\{}\bm{\delta}^{(i_1)}, \hdots, \bm{\delta}^{(i_n)}\h{\}}$ is a tuple of $n$ elements extracted from the multi-sample with $i_1 < i_2 < \hdots < i_n$, which gives the same solution as the original sample, that is,
\begin{equation}
    \mathcal{A}(\bm{\delta}^{(i_1)}, \hdots, \bm{\delta}^{(i_n)}) = \mathcal{A}(\bm{\delta}^{(1)}, \hdots, \bm{\delta}^{(S)}).
\end{equation}
\end{definition}%
The cardinality of the support subsample is referred to as the \textit{support size}, that is, $n \coloneqq |\hh{\mathcal{C}}|$. In the context of this paper, a scenario is said to be \textit{of support} if it is an element of the considered support subsample. A scenario that can be excluded from a support subsample without changing the solution is said to be \textit{not of support}. \h{For example, a sampled human trajectory $\b{\delta}^{(i)}$ can be excluded from the support if it does not change the robot's optimal behavior under the current set of human trajectory samples.}

The support size captures the number of scenarios necessary to hold the solution of the SP in place and is strongly correlated with its risk. This correlation can be used to derive a probabilistic guarantee on the solution of the SP using only the support and sample size. Denoting the \textit{confidence} as $1 - \beta$, Theorem 1 in \cite{campi_general_2018} provides the following bound
\begin{equation}%
   \mathbb{P}^{\textrm{S}}[V(\b{\theta}^*) > \epsilon(n)] \leq \sum_{n = 0}^{S - 1} {S \choose n} \left[1 - \epsilon(n)\right]^{S - n} = \beta. \label{eq:nonconvex_relation}
\end{equation}
That is, the probability that the scenario decision $\bm{\theta}^*$ exceeds the acceptable risk $\epsilon$, is upper bounded by $\beta$. 
The function $\epsilon(n) : \{0, \hdots, S\}\rightarrow [0, 1]$ is designed subject to (8) and $\epsilon(S) = 1$, \h{which divides the risk over the range of the support from $0$ to $S$.} \hh{The following mapping devides the risk evenly over all support values}
\begin{equation}
\epsilon(n) =
\begin{cases}
    1, & n > \bar{n},\\
    \h{1 - \left(\frac{\beta}{\hh{S}{S \choose n}}\right)^{\frac{1}{S-n}}}, \quad &n \leq \bar{n}.
\end{cases}\label{eq:epsilon_bounded}
\end{equation}
Notice that the violation probability increases with the support size. The more scenarios that are necessary to support a decision, the higher risk that decision is.

\h{Two limitations prevent the NSO framework from being used in online control. Firstly, t}o find a support subsample one typically has to solve $S$ optimization problems after the original problem, each time removing one of the scenarios. The set of scenarios for which the solution changes form a support subsample. This is known as the \textit{greedy} algorithm. \h{In closed loop control, solving the problem $S$ times is computationally intractable. We develop a more suitable algorithm in Sec.~\ref{sec:algorithms}.} \hh{\h{Secondly, Assumption~\ref{as:decision_algorithm} is not satisfied for distributions with unbounded support, such as Gaussian distributions, as outliers can make the problem infeasible. We consider \textit{scenario removal}, similar to the convex case (e.g.,~\cite{campi_sampling-and-discarding_2011},~\cite{garatti_risk_2019}) in Sec.~\ref{sec:removal}. Scenario removal ensures feasibility and reduces the variance of the planned trajectory over consecutive iterations.}}

\section{COMPUTING A SAFE SAMPLE SIZE}\label{sec:algorithms}
In the following, we propose an algorithm for estimating the support of an SP during optimization. We then show how this online support estimation can be \hh{used to determine if the solution of the scenario program is certifiably safe.} In the main application of this paper, our results determine for how many sampled road-user trajectories safety must be guaranteed to attain a specified risk. We establish this relation by identifying the sampled trajectories that affect the motion plan.

\subsection{\hh{Computing the Sample Size Before Optimization}}\label{sec:compute_sample_size}
\hh{First, let us consider \eqref{eq:nonconvex_relation} and \eqref{eq:epsilon_bounded} in a different perspective where the sample size $S$ is computed before optimization for some support $\bar{n}$ that we expect not to exceed. We can formulate the following result which will be useful for online applicability of the proposed method.}
\hh{\begin{theorem}\label{theory:safety-guarantee}
Suppose that Assumption \ref{as:decision_algorithm} holds true, and set a desired $\epsilon \in [0, 1]$, $\beta \in (0, 1)$. Given a support limit $\bar{n}$ and a sample size $S$ computed from \eqref{eq:nonconvex_relation} such that $\epsilon(\bar{n}) \leq \epsilon$, with $\epsilon(n)$ as in \eqref{eq:epsilon_bounded}. Then the solution of Problem 3 computed by decision algorithm $\mathcal{A}$ is feasible for Problem 2 for any $n \leq \bar{n}$.
\end{theorem}
\begin{proof}
    Under Assumption \ref{as:decision_algorithm} and for $\epsilon(n)$ as in \eqref{eq:epsilon_bounded}, \cite[Theorem 1]{campi_general_2018} ensures that~\eqref{eq:nonconvex_relation} holds. Since $\epsilon(n)$ in~\eqref{eq:epsilon_bounded} is monotonically increasing in $n$ (i.e., $\epsilon(n + 1) > \epsilon(n)$, $\forall n < S$) and given that the computed $S$ ensures that $\epsilon(\bar{n}) \leq \epsilon$, we have that $\epsilon(n) \leq \epsilon, \forall n \leq \bar{n}$. 
\end{proof}}

\hh{This shows that, for a precomputed $S$, the solution to the SP in Problem 3 can be used to solve the CCP in Problem 2, as long as $n \leq \bar{n}$.}

\subsection{Support of Iterations} \label{sec:optimization}%
\hh{We now consider how to compute the support $n$ for a nonconvex optimization problem.} The support of a convex SP can easily be computed as, under convexity, support constraints are satisfied with equality, i.e., support constraints are active~\cite{calafiore_scenario_2006, garatti_risk_2019}. This property is lost in the nonconvex case. We aim here to reestablish this property by considering, instead of the support of the decision algorithm as a whole, the support of its iterations. With \textit{iteration} we refer to the procedure that is repeated to solve the nonconvex optimization, such that the decision algorithm can be described by a repeated sequence of iterations $\mathcal{A}^l \h{: \Theta \times \Delta^S \to \Theta,} \ \h{l \in 0, \hdots, L}$ as
\begin{equation}
    \h{\mathcal{A} = \mathcal{A}^L(\hdots \mathcal{A}^1(\mathcal{A}^0(\b{\theta}^0, \b{\omega}),\b{\omega}) \hdots, \b{\omega})}
    .\label{eq:separate_algorithm}
\end{equation}
We define the support of an iteration as the set of scenarios that changes its outcome when \h{excluded}. To connect the support of the decision algorithm with that of its iterations, consider the following lemma.

\begin{lemma}\label{lemma:separation}
Consider a decision algorithm $\mathcal{A}$, separated according to \eqref{eq:separate_algorithm}. Its support $\hh{\mathcal{C}}$ satisfies
\begin{equation}
    \hh{\mathcal{C}} \subseteq  \bm{\omega} \ \textrm{\textbackslash} \ \bm{\omega}_{\textrm{ns}}, \quad \bm{\omega}_{\textrm{ns}} = \bigcap_{l = 0}^L  \bm{\omega}^l_{\textrm{ns}},\label{eq:support_property}
\end{equation}
where 
\begin{equation}
    \bm{\omega}_{\textrm{ns}}^l = \left\{\bm{\delta}^{(i)} \in \bm{\omega} \ | \ \mathcal{A}^l(\bm{\theta}^l, \bm{\omega}) = \mathcal{A}^l(\bm{\theta}^l, \bm{\omega} \textrm{\textbackslash} \bm{\delta}^{(i)}) \right\}, \label{eq:subalgorithm_property}
\end{equation}
is the set of scenarios not of support in iteration $l$.
\end{lemma}

\begin{proof}
Scenarios in $\bm{\omega}_{\textrm{ns}}$ can be excluded for all $l$ without changing the solution, that is, by \eqref{eq:support_property} and \eqref{eq:subalgorithm_property},
\begin{equation}
\mathcal{A}(\b{\theta}^0, \bm{\omega} \textrm{\textbackslash} \bm{\delta}^{(i)}) = \mathcal{A}(\b{\theta}^0, \bm{\omega}) = \b{\theta}^*,
\end{equation}
for all $\bm{\delta}^{(i)}\in\bm{\omega}_{\textrm{ns}}$. Therefore, by Definition~\ref{def:support} all scenarios in $\bm{\omega}_{\textrm{ns}}$ are \textit{not} of support for $\mathcal{A}$. The support of $\mathcal{A}$ is therefore in the complement of this set with respect to the set $\bm{\omega}$ and the result follows.
\end{proof}
The support set obtained through Lemma \ref{lemma:separation} is an overestimation. It is possible that a scenario changes the solution of an intermediate iteration without changing the final solution.

\subsection{Solving the SP Through Convex Iterations}\label{sec:convex_iterations}
We now consider the general SP in \probref{prob:general_sp}. For this problem, a local optimum can be computed by iteratively linearizing the problem and solving a convex optimization. This makes each iteration a convex scenario optimization\footnote{\hh{For example, in Sequential Quadratic Programming (SQP), each iteration $\mathcal{A}^l$ refers to the inner QPs.}} for which the support constraints are active~\cite{campi_exact_2008}. \hh{Formally we pose the following two assumptions.}
\begin{assumption}\label{as:convex_iterations}
    Each iteration $\mathcal{A}^l$ of decision algorithm $\mathcal{A}$ is convex.
\end{assumption}
\begin{assumption}\label{as:inequality_feasibility}
    \hh{The solution computed by each iteration $\mathcal{A}^l$ is feasible with respect to its inequality constraints.}
\end{assumption}
\hh{The constraints of each iteration are usually linearized scenario constraints (possibly constructed with linearized robot dynamics). It is sufficient to require that these constraints are satisfied since the active linearized scenario constraints still identify the scenarios that affect the solution.} %
\hh{\begin{theorem}\label{theory:support_estimate}
    If iterations $\mathcal{A}^l$ of the decision algorithm satisfy Assumptions \ref{as:convex_iterations} and \ref{as:inequality_feasibility}, then the support of Problem \ref{prob:general_sp} satisfies
    \begin{equation}
    \hh{\mathcal{C}} \subseteq\bigcup_{j = 0}^L \b{\omega}_{\textrm{active}}^j = \hat{\hh{\mathcal{C}}}, \quad \hat{n} = |\hat{\hh{\mathcal{C}}}|,\label{eq:support}
    \end{equation}
    where, with $g^l$ the inequality constraints of $\mathcal{A}^l$,
\begin{equation}
    \b{\omega}_{\textrm{active}}^l = \{\b{\delta}^{(i)} \in \b{\omega} \ | \ \exists k \ g^l(\b{x}^l_k, \b{\delta}_k^{(i)}) = 0\},\label{eq:active_constraints}
\end{equation}
denote the active constraints in iteration $l$.
\end{theorem}
\begin{proof}
    Under Assumption \ref{as:convex_iterations}, the support constraints of iteration $l$ are in the set $\omega_{\textrm{active}}^l$ (all constraints are satisfied by Assumption \ref{as:inequality_feasibility}). Under convexity, $\omega_{\textrm{active}}^l$ is therefore the complement of the set $\omega_{\textrm{ns}}^l$. Invoking Lemma 1 and using De Morgans Law [33], we obtain \eqref{eq:support}. 
\end{proof}}
\hh{Theorem~\ref{theory:support_estimate} estimates the support through the aggregated set of active scenarios across all iterations.} Assumption~\ref{as:convex_iterations} (convexity of iterates) and \hh{Assumption}~\ref{as:inequality_feasibility} (feasibility of intermediate iterates) restrict the solvers that can be used with the proposed approach. \hh{However, widely available solvers, such as Sequential Quadratic Programming (SQP)~\cite[Chapter 18]{nocedal_numerical_2006}, satisfy these criteria.}

\h{\subsection{Illustrating Example: SP for 1-D Obstacle Avoidance}}\label{sec:illustrating_example}
In the following, we illustrate the proposed support estimation on a simplified example. We consider a \systems following Euler-discretized unicycle dynamics, for $dt = 0.2$s,
\h{\begin{equation}
    \begin{aligned}
        \b{p}_{k+1} &= \b{p}_k + v_k \begin{bmatrix}\cos{(\eta_k)} \\ \sin{(\eta_k)}\end{bmatrix}\ dt\\
        \eta_{k+1} &= \eta_k + \omega_k \ dt.
    \end{aligned}
\end{equation}}%
\h{The environment contains a single obstacle with the same $x$ position as the robot, but uncertain $y$ position,}
\h{\begin{equation}
   p_{k+1}^{y, \textrm{obs}} = p_k^{y, \textrm{obs}} + \zeta_k, \ \zeta_k \sim \mathcal{N}(-1.3, 0.07),
\end{equation}}%
\h{in which the uncertainty at time $k$ is given by $\delta_k = \zeta_k$.} \h{The robot is subject to state constraints $0\leq v_k \leq 2$ and $-2.0\leq \omega_k \leq 2.0$. The goal for the robot is to stay above the obstacle along the horizon of $N = 5$ steps, which we enforce with a horizontal collision avoidance constraint, while minimizing its deviation in the $y$-direction. The sum of robot and obstacle radius is $r = 1.0$m. We formulate the following CCP:}
\h{\begin{subequations}
\label{eq:illustrative_example}
\begin{align}
    \min_{\b{u}\in \mathbb{U}, \b{x}\in\mathbb{X}} \quad & \sum_{k=1}^N (p_k^y)^2 + \omega_k^2 + (v_k - 2.0)^2\\
  \textrm{s.t.} \quad \quad & \b{x}_{k + 1} = f(\b{x}_k, \b{u}_k), \: \forall k, \label{eq:example_equality}\\
  &\mathbb{P}\left[\bigwedge_{k=1}^N (p^y_{k} \geq \delta_k + r) \right] \geq 1 - \epsilon, \ \b{\delta} \in \Delta.
\end{align}
\end{subequations}}%
\h{The associated SP (see Fig.~\ref{fig:toy_0}) is given by}
\begin{figure}[t]
     \centering
     \begin{subfigure}[b]{0.20\textwidth}
         \centering
         \includegraphics[width=\textwidth, trim={1cm 0cm 1cm 1cm},clip]{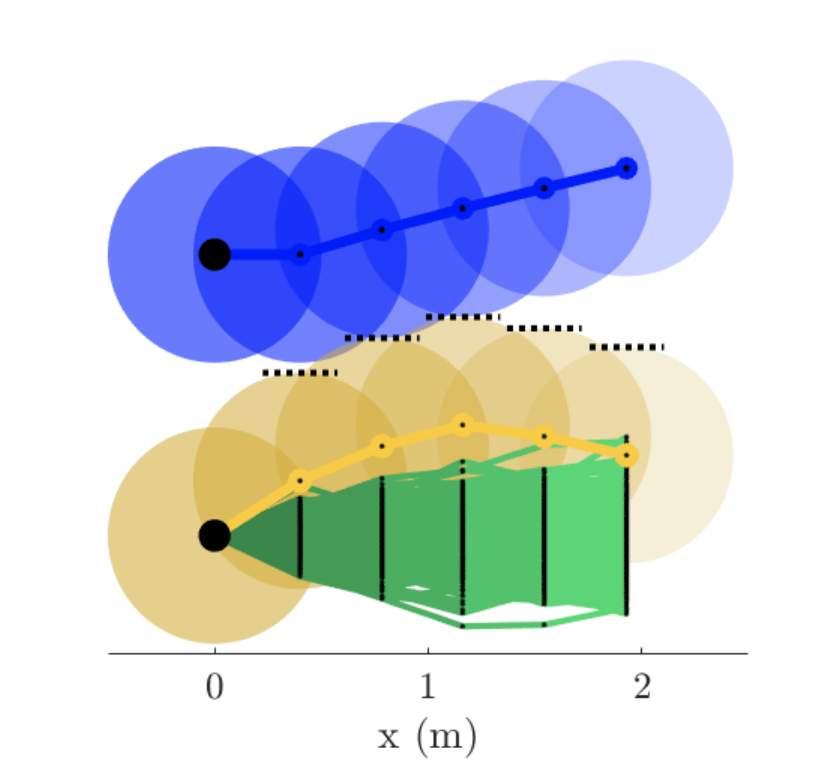}
         \caption{No Removal}
         \label{fig:toy_0}
     \end{subfigure}
     \hspace{0.005\textwidth}
     \begin{subfigure}[b]{0.20\textwidth}
         \centering
         \includegraphics[width=\textwidth, trim={1cm 0cm 1cm 1cm},clip]{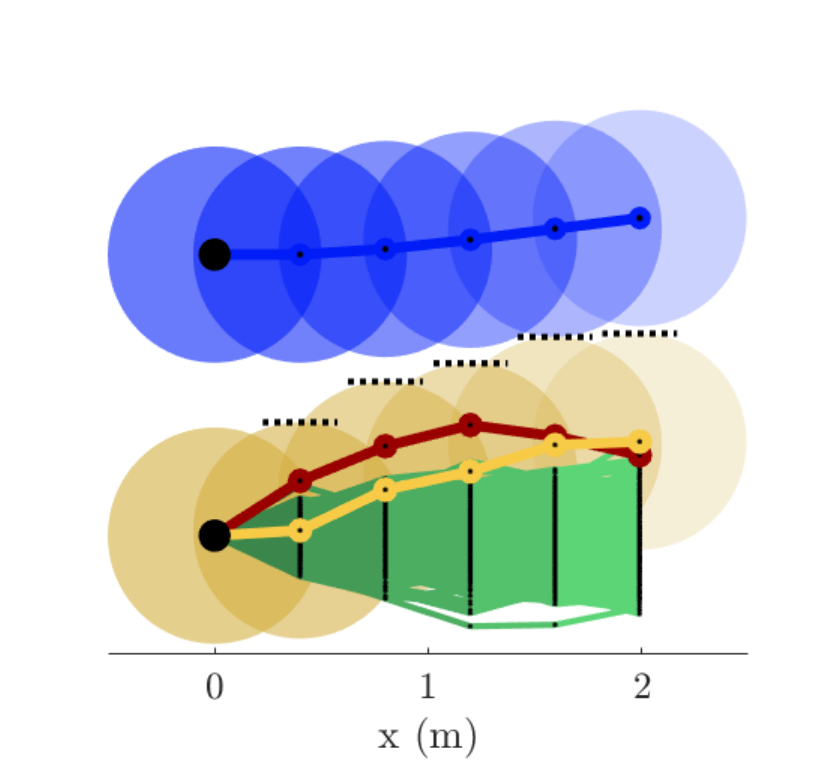}
         \caption{Removed 1}
         \label{fig:toy_8}
     \end{subfigure}
        \caption{\h{The optimized robot trajectory (blue) for the SP in Eq.~\ref{eq:illustrative_example} and the sampled obstacle trajectories (green) for $S = 400$. Samples of support (yellow) or that were removed (red) are highlighted.}}%
        \label{fig:illustrating_example}
\end{figure}
\h{\begin{subequations}
\label{eq:illustrative_example}
\begin{align}
    \min_{\b{u}\in \mathbb{U}, \b{x}\in\mathbb{X}} \quad & \sum_{k=1}^N (p_k^y)^2 + \omega_k^2 + (v_k - 2.0)^2\\
    \textrm{s.t.} \quad \quad & \b{x}_{k + 1} = f(\b{x}_k, \b{u}_k), \: \forall k,\label{eq:example_equality}\\
    &p^y_{k} \geq \delta_k^{(i)} + r, \ \forall k, \ \b{\delta}^{(i)} \in \b{\omega}, \ i = 1 , \hdots, S,
\end{align}
\end{subequations}}%
\h{which is nonconvex through the equality constraints \eqref{eq:example_equality}. We do not consider the risk $\epsilon$ achieved by \eqref{eq:illustrative_example} in this example. Instead, we compare greedy support estimation~\cite{campi_general_2018} with ours (Eq.~\ref{eq:support}) by solving $25$ realizations of SP \eqref{eq:illustrative_example} \h{using SQP} for each of the sampling sizes $S = \{100, 200, \hdots, 1000\}$. The proposed support estimate is larger than the greedy support estimate in $8/250$ experiments and smaller in $1$ experiment (i.e., it is slightly more conservative)\footnote{We note that both support estimations are not necessarily minimal.}. However, the computation times, as visualized in Fig.~\ref{fig:toy_example_runtime} indicate that the proposed method is faster and scales more favorably. For example, with $1000$ samples it requires approximately $13.2$\% of the computation time, improving further as the sample size increases.}

\hh{\subsection{The Support Limit and Associated Guarantees}\label{sec:support_limit}
We recall that, through Theorem~\ref{theory:safety-guarantee}, the risk of the planned trajectory is certified by the support of the decision algorithm with which Problem~\ref{prob:general_sp} is solved and for which Theorem~\ref{theory:support_estimate} provides an estimate. If we solve Problem~\ref{prob:general_sp} and find that $\hat{n} > \bar{n}$, then Theorem~\ref{theory:safety-guarantee} does not provide guarantees at the specified $\epsilon$. There are two reasons why this could happen. Either the CCP is infeasible for the desired $\epsilon$ (including conservatism of the scenario approach) or there may exist a safe solution under the specified $\epsilon$, but under the applied values of $\bar{n}$ and the associated $S$ the SP did not find it.}

\hh{In the first case, nothing can be done to provide a safe solution. The second case is more problematic because a feasible solution exists. Several methods can be used to deal with this case, for example, running the SP for increasing sample sizes (similarly to~\cite{calafiore_repetitive_2017}) or running several SPs in parallel (similarly to~\cite{mustafa_probabilistic_2023}). We consider these options future work and apply in this paper a constant value of $\bar{n}$ (and $S$) that is high enough that in practice Theorem~\ref{theory:safety-guarantee} holds for almost all iterations in our experiments. This approach can be conservative, as in many iterations the support may be smaller than our bound.} 

\begin{figure}[t]
    \centering
    \includegraphics[width=0.40\textwidth,trim={0.8cm, 0, 0.5cm, 0},clip]{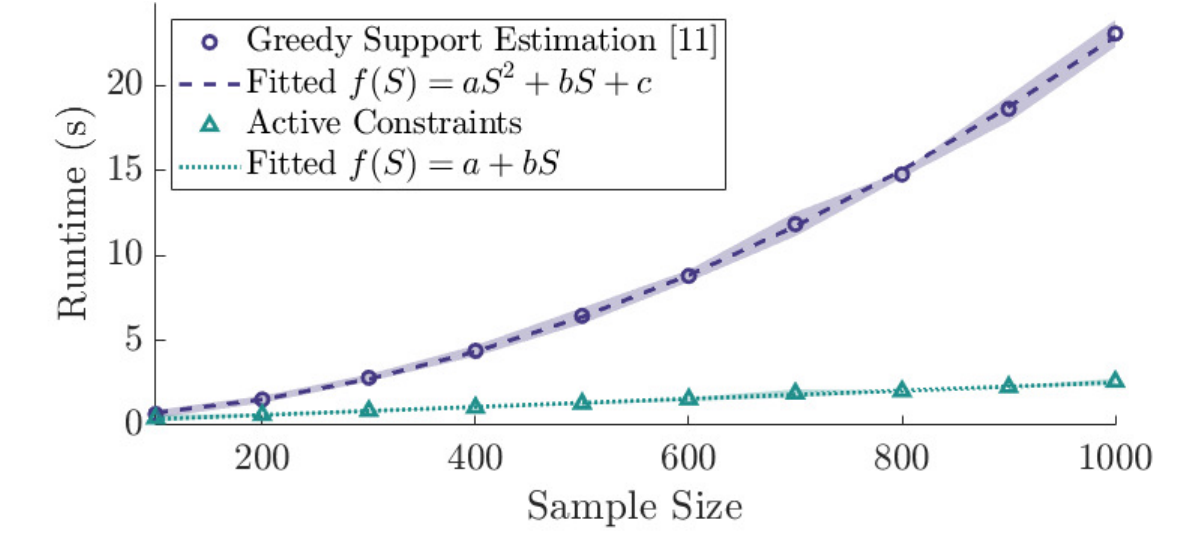}
    \caption{Mean (markers) and standard deviation (shaded region) of support estimation computation times.}
    \label{fig:toy_example_runtime}
\end{figure}%

\hh{We leverage two ideas to keep the support (and sample size) of the SP small. First, we formulate scenario constraints (see Sec.~\ref{sec:motion_planning}) such that other scenarios may become \textit{redundant}, meaning that the optimization can be described without them. Redundant samples cannot be of support and therefore reduce the complexity of the problem. For example, when sampling human trajectories, only the most extreme cases affect the motion plan. Appendix~\ref{ap:shadow} provides mathematical and empirical insights for this idea. Second, we use the support estimate from Theorem~\ref{theory:support_estimate} to evaluate during optimization if the optimization problem should continue iterating, or should be stopped to provide a suboptimal, but safe solution. The support estimate in iteration $l$ is given by}
\begin{equation}
    \hh{\mathcal{C}}^l \subseteq \bigcup_{j = 0}^l \bm{\omega}^j_{\textrm{active}} = \hat{\hh{\mathcal{C}}}^l , \quad \hat{n}^l = |\hat{\hh{\mathcal{C}}}^l|.\label{eq:support_l}
\end{equation}

\hh{\begin{theorem}\label{theory:termination}
    Suppose that that in iteration $l$ of $\mathcal{A}$, $\hat{n}^l > \bar{n}$. If, in at least one previous iteration $m < l$, all constraints of Problem~\ref{prob:general_sp} were satisfied and $\hat{n}^m \leq \bar{n}$. Then the solution $\theta^m$ is a feasible solution with support $n^m \leq \bar{n}$.
\end{theorem}
\begin{proof} 
    The solution at iteration $m$ is feasible given that all constraints of Problem~\ref{prob:general_sp} are satisfied. Terminating the optimization in principle adds to the support since it influences the outcome of the decision algorithm. However, since the decision algorithm with termination is equivalent to a decision algorithm with constant iteration length (with $L=m$), this support does not need to be included. Hence, the support is estimated by $\hat{n}^m$ and by Theorem~\ref{theory:support_estimate}, $\hat{n}^m \geq n^m$, such that $n^m \leq \bar{n}$.
\end{proof}}
\hh{Because Theorem~\ref{theory:termination} needs the iterates to be feasible before the support limit is exceeded, we cannot ensure via termination that $n \leq \bar{n}$. In practice, termination may reduce the support.} %

\subsection{Summary}
We modeled the planning problem with a constraint on the joint probability of collision (\probref{prob:ccp}) as a specific case of \probref{prob:general_ccp}. This more general CCP can be solved via the SP in \probref{prob:general_sp} as long as the sample size \h{(e.g., number of sampled road-user trajectories)} is sufficiently large. \hh{Theorem~\ref{theory:safety-guarantee} showed that for a constant sample size, the motion plan can be certified by the support \h{(i.e., number of samples that influence the plan)}.} \hh{When iterations of this algorithm are convex, Theorem~\ref{theory:support_estimate} provides a cheap support estimate that is computed during optimization.} \hh{This support estimate, through Theorem~\ref{theory:termination}, could be used to reduce the support of the optimization.} 
Through this technique \probref{prob:ccp} can be solved efficiently.

\section{Scenario Removal}\label{sec:removal} 
\hh{So far we have shown under what conditions the solution to the SP in Eq.~\ref{eq:scenario_program} solves the CCP in Eq.~\ref{eq:chance_constrained_problem} if Assumption~\ref{as:decision_algorithm} is satisfied. In this section, we consider how Assumption~\ref{as:decision_algorithm} can be satisfied, particularly for distributions with unbounded support. The SP in Eq.~\ref{eq:scenario_program} does not satisfy this assumption outright. For example, in the motion planning application there is always a non-zero probability that a sampled pedestrian overlaps with the robot, making the SP infeasible even when the CCP admits a feasible solution. We introduce in this section a modified SP with \textit{scenario removal}~\cite{garatti_risk_2019} to satisfy Assumption~\ref{as:decision_algorithm} under unbounded support. Scenario removal allows the SP to remove limiting samples (potential outliers) after sampling, at the cost of drawing more samples initially.}

\hh{Scenario removal additionally reduces the variance of the solution to the SP. This makes the solution more consistent over consecutive iterations. In a motion planning context, scenario removal makes it possible to ignore extremely unlikely road-user trajectories that lead to too conservative motion plans under the specified risk.}

\subsection{Scenario Removal SP}
\hh{Consider a \textit{removal algorithm} $\mathcal{R}$} that identifies scenarios to be removed from the SP in \probref{prob:general_sp} with the goal to reduce the cost \hh{or to make the problem feasible}. That is, a function $\mathcal{R} : \Theta \times \Delta^S \to \Delta^r$,
\begin{equation}
    \mathcal{R}(\bm{\theta}, \bm{\omega}) = (\bm{\delta}^{(i_1)}, \hdots, \bm{\delta}^{(i_r)}).\label{eq:removal_function}
\end{equation}
The intent is that constraints indexed by the removal function are \textit{removed} from the scenario constraints in \probref{prob:general_sp} when solving the SP (Eq.~\ref{eq:scenario_program}). \hh{With regards to Assumption~1 we pose the following assumption on $\mathcal{R}$.}

\begin{assumption}\label{as:removal_feasibility}
    \hh{The removal algorithm removes any scenarios that prevent a feasible solution. Formally, Problem \ref{prob:general_sp}, where \eqref{eq:scenario_constraints} excludes removed scenarios as
\begin{equation}
    g(\b{x}, \b{\delta}^{(i)})\leq 0, \b{\delta}^{(i)} \in \b{\omega} \backslash\mathcal{R}(\b{\theta}, \b{\omega}),
\end{equation}
is assumed to be feasible.}
\end{assumption}
\hh{We now show how the SP in Problem~\ref{prob:general_sp} can be modified under scenario removal.}
\hh{\begin{theorem}\label{theory:removal}
The SP in Problem~\ref{prob:general_sp}, including a scenario removal algorithm $\mathcal{R}$ satisfying Assumption~\ref{as:removal_feasibility} is a decision algorithm that satisfies Assumption \ref{as:decision_algorithm}.
\end{theorem}}
\begin{proof}
Inspired by \cite[Chapter 5]{garatti_risk_2019}, we may equivalently represent the SP with removal by extending the decision variables with $S$  Boolean variables $\b{\xi} = \begin{bmatrix} \xi^{(1)}, \hdots, \xi^{(S)}\end{bmatrix}$ , $\b{\xi} \in \{0, 1\}^S$, such that ${\b{\theta}} = (\b{x},\b{u}, \b{\xi})$ and where each boolean variable is $1$ if and only if the scenario it indexes is contained in $\mathcal{R}$. We then reformulate the original SP in Problem~\ref{prob:general_sp} as 
    \begin{problem}[SP with Removal]\label{prob:removal_sp}
        \begin{subequations}
            \begin{align}
            \min_{(\b{{x}}, \b{{y}}, \b{{\xi}})\in\Theta} \quad & \sum_{k = 0}^N J(\b{{x}}_k, \b{{u}}_k)\\
            \textrm{s.t.} \qquad & \b{{x}}_0 = \b{{x}}_{\textrm{init}}\\
            & \b{{x}}_{k + 1} = f(\b{{x}}_k, \b{{u}}_k, \b{{\delta}}_k^{(i)})\\
            &g(\b{{x}}, \b{{\delta}}^{(i)}) \leq {\xi}^{(i)} M, \ \b{{\delta}}^{(i)} \in \b{{\omega}}, \forall i,
            \end{align}
        \end{subequations}
    \end{problem}
    The key difference is that the constraint formulation $g(\b{{x}}, \b{{\delta}}^{(i)})\leq 0$ is modified to $g(\b{{x}}, \b{{\delta}}^{(i)}) \leq {\xi^{(i)}}M$, where a big M formulation ($M \gg 0$) removes the scenario constraint when $\xi^{(i)} = 1$.
    
    \hh{To show that Problem~\ref{prob:removal_sp} satisfies Assumption \ref{as:decision_algorithm} (i.e., it admits a feasible solution for all possible sample extractions $\b{\omega}$), we note that under the modified formulation $\b{\theta} \in \Theta_{\delta^{(i)}}$ can hold in two ways:} the original scenario constraint is satisfied (i.e., $g(\b{{x}}, \b{{\delta}}^{(i)})\leq 0$) or the scenario is removed (i.e., $\xi^{(i)} = 1$). \hh{Under Assumption \ref{as:removal_feasibility}, a feasible solution to Problem~\ref{prob:removal_sp} exists for all scenarios where $\xi^{(i)} = 0$, while scenarios where $\xi^{(i)} = 1$ are removed. Therefore, for any sample extraction $\b{\omega}$, $\mathcal{A}(\b{\omega}) \in \Theta_{\delta^{(i)}}, \forall i = 1, \hdots, S$.}
\end{proof}
Because the removal of constraints may increase the risk\h{\footnote{\h{Note that removal of constraints is different than excluding constraints that are not of support. Removing constraints changes the solution, adds support and increases the risk.}}}, the sample size of the SP must be increased to guarantee that the risk remains the same while some scenarios are removed after sampling. We derive the following result.
\hh{\begin{theorem}\label{theory:removal_support}
The support of Problem~\ref{prob:removal_sp} includes all removed scenarios, that is,
\begin{equation}
    \mathcal{R}(\b{\theta}, \b{\omega}) \subseteq \hh{\mathcal{C}}.
\end{equation}
\end{theorem}
\begin{proof}
    Scenarios removed by $\mathcal{R}$ hold the solution over the variables $\b{\xi}$ in place, since excluding a removed scenario $i$ from $\b{\omega}$, would lead to $\xi^{(i)} = 0$ instead of $\xi^{(i)} = 1$. Thus, by definition, all removed scenarios must be included in the support of Problem~\ref{prob:removal_sp}.
\end{proof}}

\h{Although $\b{\xi}$ are incorporated \hh{in the optimization problem} in Theorem~\ref{theory:removal}, \hh{the scenarios to be removed} can be determined before or during optimization through $\mathcal{R}$ such that no mixed integer optimization needs to be solved. Rather, we solve the original SP in Problem~3 without the scenarios indexed by $\mathcal{R}$ and Theorem~\ref{theory:removal_support} shows how \hh{Theorem~\ref{theory:safety-guarantee}} can be applied to guarantee probabilistic safety.}

\begin{figure}[b]
    \centering
    \includegraphics[width=0.46\textwidth, trim={0.5cm 0 0cm 0cm}, clip]{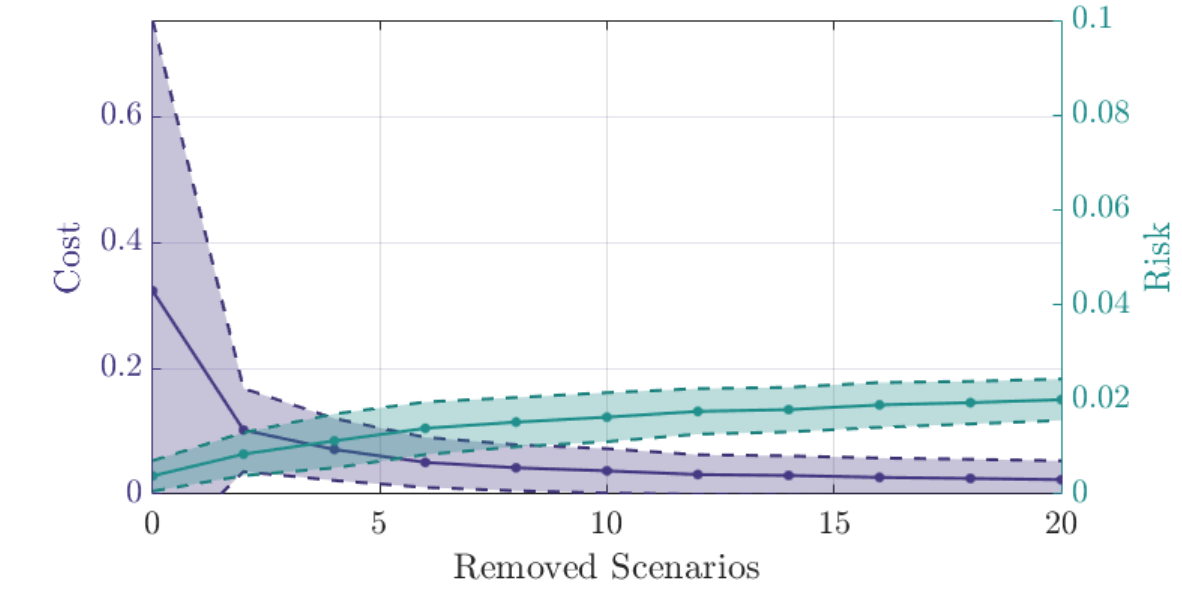}
    \caption{\h{Cost (purple) and empirical risk (green) distributions for the illustrating example \ho{in Section~\ref{sec:illustrating_contd}}. The sample size for each removal size is computed using Theorem~\ref{theory:removal}. Colored intervals denote the variance over $100$ repeated experiments.}}
    \label{fig:cost_vs_risk_revision}
\end{figure}%

\subsection{Illustrating Example Continued}\label{sec:illustrating_contd}
\hh{To illustrate that scenario removal improves consistency of the SP}, we continue the illustrating example. Fig.~\ref{fig:toy_8} depicts a solution to the example where \h{one scenarios is} removed, \h{which decreases the trajectory cost} compared to Fig.~\ref{fig:toy_0}. We perform a quantitative analysis on scenario removal, solving the SP in Eq.~\ref{eq:illustrative_example} \h{for $R \in \{0, 2, \hdots, 20\}$ while updating the sample size according to Theorem~\ref{theory:removal_support} with $\epsilon = 0.1, \beta = 10^{-6}$, resulting in $S\in\{290, 390, \hdots, 1250\}$. We solve the problem $100$ times for each $R$. Fig.~\ref{fig:cost_vs_risk_revision} visualizes the mean and variance of the resulting cost and the empirical risk, validated by monte-carlo sampling with $10^4$ samples.}

\h{Because outliers are removed, the cost decreases in value and variance for any non zero removal size. For larger removal sizes, the variance of the cost decreases mildly with diminishing returns. The risk is consistently less conservative with removal. }

\h{\section{SAFE HORIZON MODEL PREDICTIVE CONTROL}\label{sec:shmpc}}
\hh{So far, we have posed the robot navigation problem as a CCP under joint chance constraints (\probref{prob:ccp}), a specific case of \probref{prob:general_ccp}, that we can solve via an SP (\probref{prob:removal_sp}). In Sec.~\ref{sec:algorithms}, we showed how the safety of the solution computed by the SP can be assessed in online control. In this section, we describe how the proposed framework can be applied in closed loop for safe robot navigation.}

\begin{algorithm}[t]
\caption{Safe Horizon MPC}
\label{alg:sqp_procedure}
\textbf{Input: } $\epsilon$, $\beta$, $\bar{n}$, $R$ and $\mathbb{P}$\\
$S \ \leftarrow $ Bisection of Eq.~\ref{eq:epsilon_bounded} using $\epsilon, \beta, \bar{n}$

\While{True}{
    $\hh{\b{\omega}} \: \leftarrow$ Sample $\mathbb{P}$\\

    \For{$l = 1, \hdots, L$}{
        $\bm{\theta}^l \qquad \leftarrow \ $Solve an iteration of the SP (Eq.~\ref{eq:scenario_program})
        
        $\bm{\omega}_{\textrm{active}}^l \ \: \leftarrow \ $Determine active scenarios (Eq.~\ref{eq:active_constraints})

        $\hat{\hh{\mathcal{C}}}^l, \hat{n}^l \: \: \: \leftarrow \ $Aggregate the support (Eq.~\ref{eq:support_l})
        
        \If{$\hat{n}^l > \bar{n}$ \hh{and Eqs.~\eqref{eq:scenario_eq_constraints} and~\eqref{eq:scenario_constraints}} \hh{satisfied at $l=m$}}{
            $\bm{\theta}^* = \bm{\theta}^{m}$
            
            \textbf{break}
        }
        
        $\bm{\xi}^l \ \leftarrow \ $Remove scenarios (Eq.~\ref{eq:removal_function}) %
    }
    Actuate $\b{u}^*_0$
}
\end{algorithm}%

\h{\subsection{MPC Algorithm}}
\h{We formulate and solve the Model Predictive Control (MPC) problem in Algorithm~\ref{alg:sqp_procedure}.} Before deploying the controller, we compute the sample size from the desired risk, confidence and support \hh{limit} (line~2). For completeness, we provide with this publication a Jupyter notebook~\cite{de_groot_o_jupyter_2022} that performs this computation using a bisection of \eqref{eq:epsilon_bounded}. In each control iteration, we sample the modeled distribution of the uncertainty to obtain a set of samples (line~4). Then, we repeatedly solve \probref{prob:general_sp} with the samples and aggregate the support estimate by computing the active constraints (line~6-8). If the support estimate exceeds the support \hh{limit} \hh{and there was a previously safe solution, we return that solution} (line~9-11). \hh{If there was no previous solution, we slow down the robot.} The first input is sent to the system (line~13). %

\h{\subsection{Removing Scenarios}}
\h{In line~12 \hh{we remove scenarios to improve the solution and ensure feasibility} as detailed in Sec.~\ref{sec:removal}.} Function $\mathcal{R}$ updates the removed scenarios in each iteration\footnote{$\mathcal{R}$ must be carefully selected to reduce the cost. If it does not, then the added support quantified in Theorem~\ref{theory:removal_support} results in conservative solutions.}. Removing the active constraints is an effective way to reduce the cost. These scenarios restrict the optimization, are already of support and are available without additional computations. \h{They represent for example the most restrictive sampled human trajectories in our main application.} We therefore define the removal strategy for SH-MPC as
\begin{equation}
    \mathcal{R}(\bm{\theta}, \bm{\omega}) = \bigcup_{j = 0}^l \hh{\left(\bm{\omega}^j_{\textrm{active}} \cup \b{\omega}^j_{\textrm{infeasible}}\right)}, \label{eq:removal_algorithm}
\end{equation}
\hh{where $\b{\omega}^j_{\textrm{infeasible}}$ denote the scenarios that must be removed to satisfy Assumption~\ref{as:removal_feasibility}. Active constraints $\bm{\omega}^j_{\textrm{active}}$ can be removed up to some maximum of $R$ scenarios. Accordingly, the support increases by at least $R$ through Theorem~\ref{theory:removal_support}.}

Even when removed scenarios are already of support, as is the case in \eqref{eq:removal_algorithm}, the support of the optimization still increases. Removing scenarios causes underlying scenarios to become of support as illustrated in Fig.~\ref{fig:removal_figure}. Removing $2$ scenarios in Fig.~\ref{fig:removal_a} leads to Fig.~\ref{fig:removal_b}. The cost decreases, but all $4$ scenarios are of support, making its solution higher risk.

\begin{figure}[t!]
    \centering
    \begin{subfigure}[t]{0.20\textwidth}
    \centering
         \includegraphics[width=\textwidth]{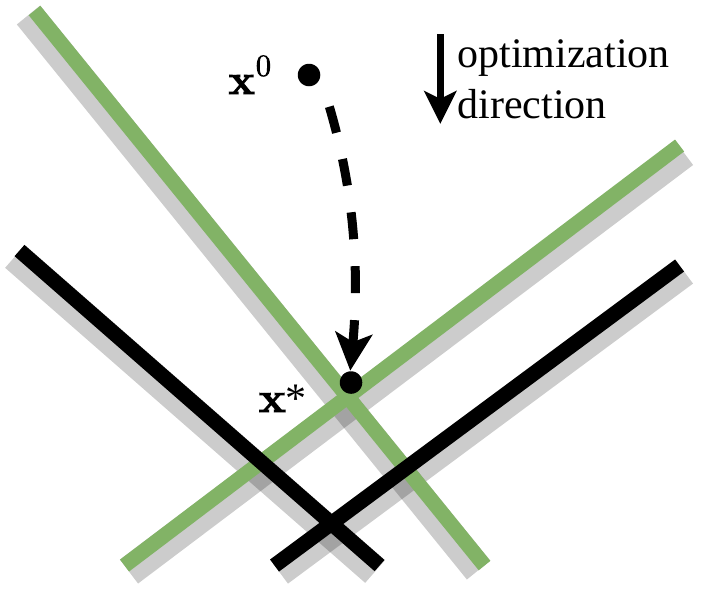}
            \caption{No scenarios are removed, the support is $2$.}%
         \label{fig:removal_a}
     \end{subfigure}
     \quad
     \begin{subfigure}[t]{0.20\textwidth}
         \centering
         \includegraphics[width=\textwidth]{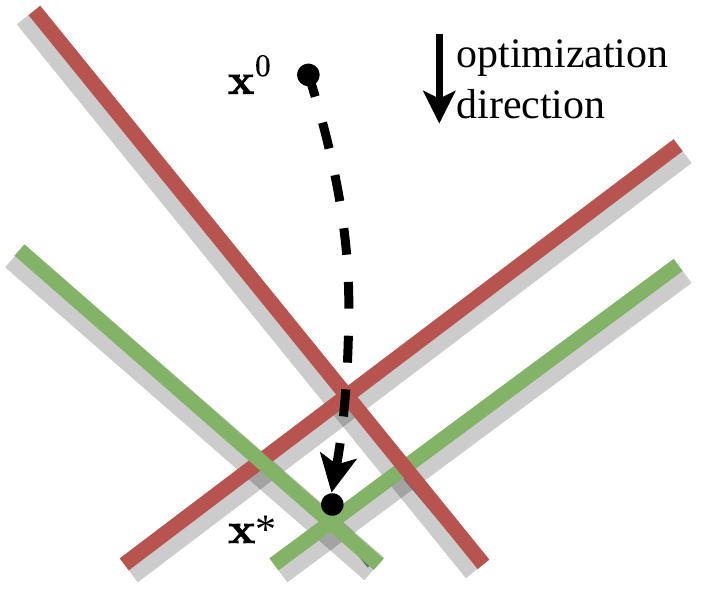}
        \caption{After removing 2 scenarios (linear constraints), the cost decreases but all 4 scenarios are in the support set.}%
         \label{fig:removal_b}
     \end{subfigure}
    \caption{An illustration of scenario optimization without (left) and with (right) scenario removal. The dots denotes the initial guess and solution to the optimization. The thick lines depict inactive (black), active (green) and removed (red) constraints.}
    \label{fig:removal_figure}
\end{figure}

\begin{figure*}
    \centering
     \begin{subfigure}[t]{0.30\textwidth}
         \centering
         \includegraphics[angle=-90,width=\textwidth]{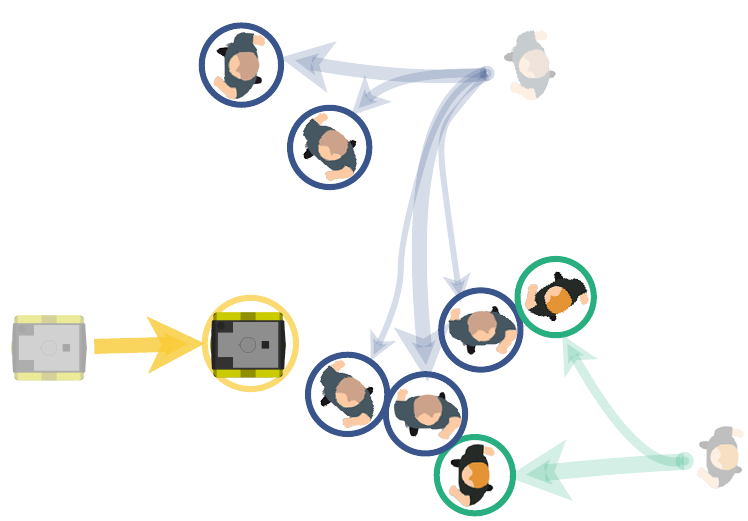} %
        \caption{Scenarios are sampled from the trajectory distributions. Each time instance of SP (Eq.~\ref{eq:quadratic_sp}) is associated with a set of sampled obstacle positions as visualized by the green and blue circled pedestrians.}
         \label{fig:overall_b}
     \end{subfigure}
    \qquad
     \begin{subfigure}[t]{0.30\textwidth}
         \centering
         \includegraphics[angle=-90,width=\textwidth]{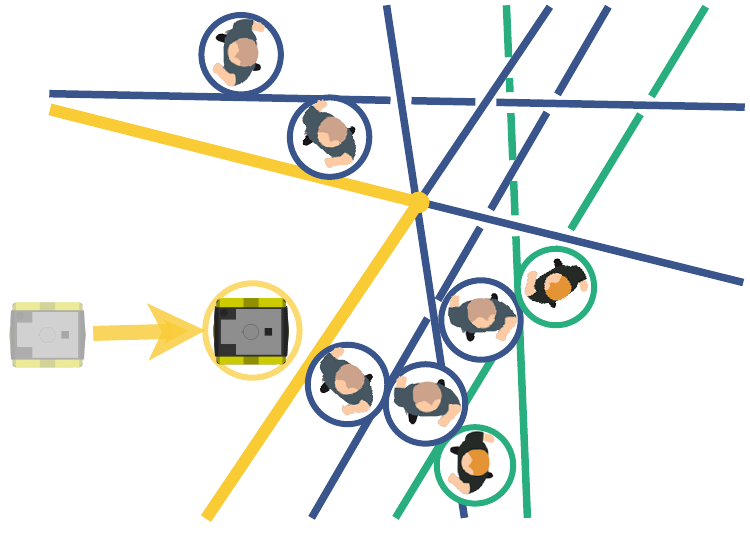}
        \caption{Linear constraints are constructed between sampled obstacles and the \system, and are reduced to a probabilistic safe polytope for each time instance and \systems disc.}
         \label{fig:overall_c}
     \end{subfigure}
        \qquad
     \begin{subfigure}[t]{0.30\textwidth}
         \centering
         \includegraphics[angle=-90,width=\textwidth]{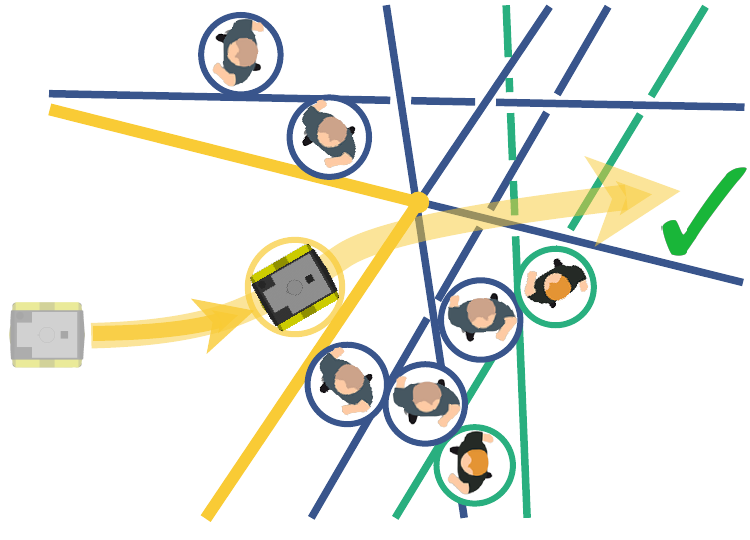}
        \caption{\h{SP (Eq.~\ref{eq:sp_final}) is solved via Algorithm~\ref{alg:sqp_procedure}. The resulting trajectory is certified up to a probabilistic bound.}}
         \label{fig:overall_d}
     \end{subfigure}
    \caption{\h{Schematic illustration of Safe Horizon MPC applied to a mobile robot.}}
    \label{fig:overall_figure}
\end{figure*}

\h{\section{APPLICATION: SAFE MOTION PLANNING AROUND PEDESTRIANS}\label{sec:motion_planning}}
\hh{In this section, we apply Safe Horizon MPC as derived in Sec.~\ref{sec:shmpc} to compute probabilistic safe trajectories in real-time for a mobile robot navigating around pedestrians.} %

\subsection{Collision Avoidance Scenario Program}
\hh{In the following, we let the joint uncertainty in~Problem~\ref{prob:ccp} represent} the future motion of all humans near the robot, that is, $\Delta = \mathbb{R}^{2MN}$, which represents the $x$-$y$ positions of $M$ obstacles over $N$ time steps. Based on obstacle tracking information, a perception module provides the joint probability distribution of obstacle trajectories $\mathbb{P}$. This distribution is sampled to obtain the scenarios that describe positions of all obstacles along the planning horizon. With these scenarios, $\h{\b{\delta}^{(i)} \in \b{\omega}}$, we can construct the SP for \probref{prob:ccp} as follows:
\begin{subequations}
\label{eq:quadratic_sp}
\begin{align}
    \min_{\b{u} \in \mathbb{U}, \b{x}\in\mathbb{X}} \quad & \sum_{k = 0}^N J(\b{x}_k, \b{u}_k)\\
  \textrm{s.t.} \quad \quad & \b{x}_0 = \b{x}_{\textrm{init}} \\
  & \b{x}_{k + 1} = f(\b{x}_k, \b{u}_k), \ k = 0, \hdots, N-1 \\
  &\bigwedge_{i = 0}^S\bigwedge_{k = 1}^N\bigwedge_{d = 0}^{n_d} \bigwedge_{j = 0}^M \left(||\b{p}_k^d - \bm{\delta}^{(i)}_{k, j}||_{\h{2}} \geq r\right),\label{eq:quadratic_sp_constraints}
\end{align}
\end{subequations}
where, for brevity, the scenario constraints are bundled using the ``and'' operator. The collision-free space described by~\eqref{eq:quadratic_sp_constraints} is nonconvex (see Fig.~\ref{fig:overall_b}) \hh{and can lead to high support.}
\hh{To reduce the support and computation time of the problem,} we first linearize the collision regions with respect to the previously planned \systems trajectory (denoted $\hat{\b{p}}$). After linearization, each scenario is associated with a linear constraint (depicted as lines in Fig.~\ref{fig:overall_c}). \hh{Details on how this reduces the support are included in Appendix~\ref{ap:shadow}.} For a previous \systems position $\hat{\b{p}}_k$ and obstacle position $\bm{\delta}_k$ the constraints are given by
\begin{equation}
    \mathcal{H}(\hat{\b{p}}_k, \bm{\delta}_k) = \{\b{p}_k \ | \ \b{A}(\hat{\b{p}}_k, \bm{\delta}_k)^T\b{p}_k \leq b(\hat{\b{p}}_k, \bm{\delta}_k)\},
\end{equation}
where
\begin{equation}
    \b{A}(\hat{\b{p}}_k, \bm{\delta}_k) = \frac{\bm{\delta}_k - \hat{\b{p}}_k}{||\bm{\delta}_k - \hat{\b{p}}_k||}, \ b(\hat{\b{p}}_k, \bm{\delta}_k) = {\b{A}}^T\bm{\delta}_k - r.\label{eq:hyperplane_definition}
\end{equation} 
The linearized collision region contains the original collision region, preserving the probabilistic guarantees. In addition, the linearization is applied locally to each timestep $k$ and \systems disc $d$, resulting in a locally accurate approximation of the original collision regions. For the linearized constraints, we obtain the following SP:
\begin{problem}[SP]\label{prob:sp}
\begin{subequations}
\label{eq:mp_scenario_program}
\begin{align}
    \min_{\b{u} \in \mathbb{U}, \b{x}\in\mathbb{X}} \quad & \sum_{k = 0}^N J(\b{x}_k, \b{u}_k)\\
    \textrm{s.t.} \quad \quad & \b{x}_0 = \b{x}_{\textrm{init}} \\
& \b{x}_{k + 1} = f(\b{x}_k, \b{u}_k), \ k = 0, \hdots, N-1 \\
  &\bigwedge_{i = 0}^S\bigwedge_{k = 1}^N\bigwedge_{d = 0}^{n_d} \bigwedge_{j = 0}^M \left(\b{p}^d_k \in\mathcal{H}\left(\hat{\b{p}}_k^d, \bm{\delta}_{k, j}^{ (i)}\right)\right).\label{eq:mp_scenario_constraints}%
\end{align}
\end{subequations}
\end{problem}
\subsection{Improved Computational Efficiency}
To further reduce the computational demand of this formulation, consider that Constraints \eqref{eq:mp_scenario_constraints} can be reordered as
\begin{equation}
      \bigwedge_{k = 1}^N\bigwedge_{d = 0}^{n_d} \left[\bigwedge_{i = 0}^S \bigwedge_{j = 0}^M \left(\b{p}^d_k \in \mathcal{H}\left(\hat{\b{p}}_k^d,\bm{\delta}_{k, j}^{(i)}\right)\right)\right], %
      \label{eq:permuted_scenario_constraints}
\end{equation}
to pair constraints that apply to a single \systems disc position $\b{p}_k^d$. Because of the overlap between the constraints, each of these constraint pairings can be described by a small subset of the constraints (see Fig.~\ref{fig:overall_c}). Thus, the constraints \eqref{eq:mp_scenario_constraints} can be reduced to a free-space polytope before optimization, which significantly reduces computation times. We denote the polytope for \systems disc $d$ and stage $k$ as
\begin{equation}
   \mathcal{P}_k^d = \left\{\b{p}_k^d  \ \bigg| \ \bigwedge_{c = 0}^{n_{\mathcal{H}}} \b{p}^d_k \in \mathcal{H}_c \right\}.\label{eq:polytope}
\end{equation}
By constructing this polytope, the number of constraints is reduced from $10^3 - 10^4$ to no more than $n_{\mathcal{H}} =20$ constraints. In addition, the computations necessary to construct the polytope are cheap in 2D. Our algorithm uses a recursive search that typically performs this computation in less than $100\:\mu$s. The final SP that is solved online is given by %
\begin{subequations}
\label{eq:sp_final}
\begin{align}
    \min_{\b{u} \in \mathbb{U}, \b{x}\in\mathbb{X}} \quad & \sum_{k = 0}^N J(\b{x}_k, \b{u}_k)\\
      \textrm{s.t.} \quad \quad & \b{x}_0 = \b{x}_{\textrm{init}} \\
& \b{x}_{k + 1} = f(\b{x}_k, \b{u}_k), \ k = 0, \hdots, N-1 \\
 &\b{p}^d_k \in \mathcal{P}_k^d, \quad \forall k, \forall d.
\end{align}
\end{subequations}
\h{Its solution is a safe trajectory as visualized in Fig.~\ref{fig:overall_d}.}

\subsection{Implementation Details}
We apply our framework in a Model Predictive Contouring Control formulation~\cite{brito_model_2019}, which tracks a reference path and reference velocity while penalizing \systems inputs\footnote{In contrast to~\cite{brito_model_2019}, our objective does not use repulsive forces from the obstacles.}. 
We solve the SP in Eq.~\ref{eq:sp_final} using the Forces Pro SQP solver~\cite{domahidi_forces_2014} with at most $15$ iterations. We reduce the variance of the solution by removing $R=\hh{1}$ scenarios through $\mathcal{R}$ as defined in Eq.~\eqref{eq:removal_algorithm}. Including the support of removal we set $\bar{n} = \hh{9}$ for which \hh{Theorem~\ref{theory:safety-guarantee} almost always certifies the solution (if not, we slow down the robot).} For a risk of $\epsilon = 0.05$ and confidence of $\beta = 0.01$ the required sample size is $S = \hh{1237}$. The linearization of the constraints requires the previous plan to be feasible. We make use of a projection step to ensure that this holds in all time steps. This projection step consists of a projection orthogonal to the direction of \systems movement, which almost always results in a feasible plan. In the remaining cases, we solve a feasibility program using Douglas-Rachford Splitting~\cite{artacho_douglas-rachford_2019}. \ho{Our motion planner is implemented in C++/ROS and will be released open source}.

\section{SIMULATION RESULTS}\label{sec:results}
This section compares our SH-MPC with baseline methods that constrain the marginal CP, that is, the independent CP per time instance and obstacle. In the first set of simulations, we consider a robot moving through an environment with pedestrians (see Fig.~\ref{fig:gaussian}) in which the \systems is modeled by a kinematic unicycle model~\cite{siegwart_introduction_2011}. We assume throughout that the distribution of pedestrian motion is known, in order to evaluate the performance of the planner in isolation (i.e., without prediction errors). We first validate on a Gaussian case, where the baselines may leverage the shape of the distribution to accurately approximate the probabilistic collision-free space. A video of the simulations accompanies this paper~\cite{o_de_groot_scenario-based_2022}.

\subsection{Baselines}
The first baseline~\cite{zhu_chance-constrained_2019} (referred to as ``Gaussian'') is strictly applicable to Gaussian distributions. It approximates the collision probability via the Cumulative Density Function (CDF) of the Gaussian distribution. The linearized collision avoidance chance constraint
\begin{equation}
    \mathbb{P}\left[\b{a}_{k, j}^T(\b{p}_k- \bm{\delta}_{k, j}) \geq r\right] \geq 1 - \epsilon_k, \ \b{a}_{k, j}\!=\!\frac{\b{p}_k - \bm{\delta}_{k, j}}{||\b{p}_k - \bm{\delta}_{k, j}||},
\end{equation}
is equivalent, under a Gaussian distribution of $\bm{\delta}_{k, j}$, to
\begin{equation}
    \b{a}_{k, j}^T(\b{p}_k - \bm{\delta}_{k, j}) - r \geq \textrm{erf}^{-1}(1 - 2\epsilon_k) \sqrt{2\b{a}_{k, j}^T\bm{\Sigma}\b{a}_{k, j}},
\end{equation}
where $\textrm{erf}^{-1}$ is the inverse standard error function and $\bm{\Sigma}$ is the covariance matrix of the uncertainty. This constraint is imposed separately for each time step and obstacle.

The second baseline S-MPCC~\cite{de_groot_scenario-based_2021} is generally applicable. It solves an SP where scenario constraints are posed on the marginal distributions. For each time $k$, it samples from the independent obstacle distribution at $k$ to obtain a collision-free polygon. It is assumed that all constraints in the polygon are of support (set to $20$ in these simulations), which for $\epsilon_k = 0.0025$ requires $S = 75946$. Samples in the center of the distribution are pruned in the Gaussian case to reduce the number of samples considered online.

Since SH-MPC is characterized by a single bound $\epsilon$ on the trajectory CP, while the baselines specify bounds $\epsilon_k$ on the CP for each $k$ and for each obstacle, we consider three versions of the baselines. The first sets $\epsilon_k = \epsilon$, which is not provably safe, but relies on updates of the controller to remain safe. The second version sets $\epsilon_k = \frac{\epsilon}{N}$, accounting for the marginal approximation over time since $\sum_k \epsilon_k = \epsilon$, but ignoring marginalization per obstacle. The third version accounts for both marginalizations, setting $\epsilon_k = \frac{\epsilon}{N M}$ such that $M \sum_k  \epsilon_k = \epsilon$. Only this last version attains the same safety guarantee as SH-MPC. We consider this version in crowded environments, where the marginal CP is violated otherwise.

\subsection{Weights and Parameters}
In the following simulations, the planners are only differentiated by their collision avoidance constraints. We use the same solver, weights and cost function for all methods. The weights of the MPC problem are given in Table~\ref{tab:weights}. To motivate overtaking behavior, the MPC is tuned differently for the two considered distributions. Aside from the weights, we define a horizon of $N = 20$ steps, with a discretization step of $0.2$ s, giving a time horizon of $4.0$ s. The control rate is $20$ Hz corresponding to a sampling time of $50$ ms. The computer running the simulations is equipped with an Intel i9 CPU@2.4GHz.

\begin{table}[h]
    \centering
    \caption{Weights of the MPC problem.}
    \begin{tabular}{|l|c|c|c|c|}
    \hline\textbf{Simulation (Sec.)} & \textbf{Velocity} & \textbf{Ang. Velocity} & \textbf{Contour} & \textbf{Lag} \\\hline
    Trajectories (\ref{sec:results_trajectories})& 0.05 & 0.05 & 0.02 & 0.1 \\\hline
    Gaussian (\ref{sec:results_gaussian}) & 0.05 & 0.05 & 0.001 & 0.1 \\\hline
    GMM (\ref{sec:results_gmm}) & 0.15 & 0.05 & 0.005 & 0.1 \\\hline
    \end{tabular}
    \label{tab:weights}
\end{table}

\subsection{Static Obstacle Trajectory Comparison}\label{sec:results_trajectories}
We first compare trajectories of SH-MPC and the Gaussian baseline in an environment with one static obstacle where the motion predictions follow a Gaussian distribution. The pedestrian dynamics are given by
\begin{equation}
\bm{\delta}_{k + 1} = \bm{\delta}_{k} + \bm{\delta}_{w, k}dt, \quad \bm{\delta}_{w, k} \sim \mathcal{N}(\b{0}, \bm{\Sigma}_{w}),\label{eq:ped_dynamics}
\end{equation}
where $\bm{\Sigma}_{w} \in \mathbb{R}^{2\times 2}$ is a diagonal covariance matrix (i.e., random variables in the $x$ and $y$ direction are independent). Its diagonal entries $\sigma_{wx} = 0.3, \sigma_{wy} = 0.3$ are kept constant over the horizon. For the marginal Gaussian baselines, the distribution of the uncertainty in future steps is propagated using
\begin{equation}
    \sigma_{x, k+1}^2 = \sigma_{x, k}^2 + \sigma_{wx}^2, \ \sigma_{y, k+1}^2 = \sigma_{y, k}^2 + \sigma_{wy}^2,
\end{equation}
while for SH-MPC we specify directly the dynamics in Eq.~\eqref{eq:ped_dynamics}. We bound the CP for SH-MPC at $\epsilon = 0.05$ and the marginal CP at $\epsilon_k = 0.05$ and $\epsilon_k = 0.0025$ ($5\%$ over $N=20$ steps) for the baselines. The obstacle is positioned at $(6.0, 0.0)$ and does not move in the simulation. The pedestrian and robot radius are $0.3$ m and $0.325$ m respectively. We tune the controller to follow the reference trajectory on the x-axis as close as possible and use the same weights for all methods (see Table~\ref{tab:weights}). The resulting trajectories over 25 repetitions are depicted in Fig.~\ref{fig:trajectories_zero}.

\begin{figure}[t]
     \centering
     \includegraphics[width=0.37\textwidth,trim={1cm, 1cm, 2.5cm, 0cm},clip]{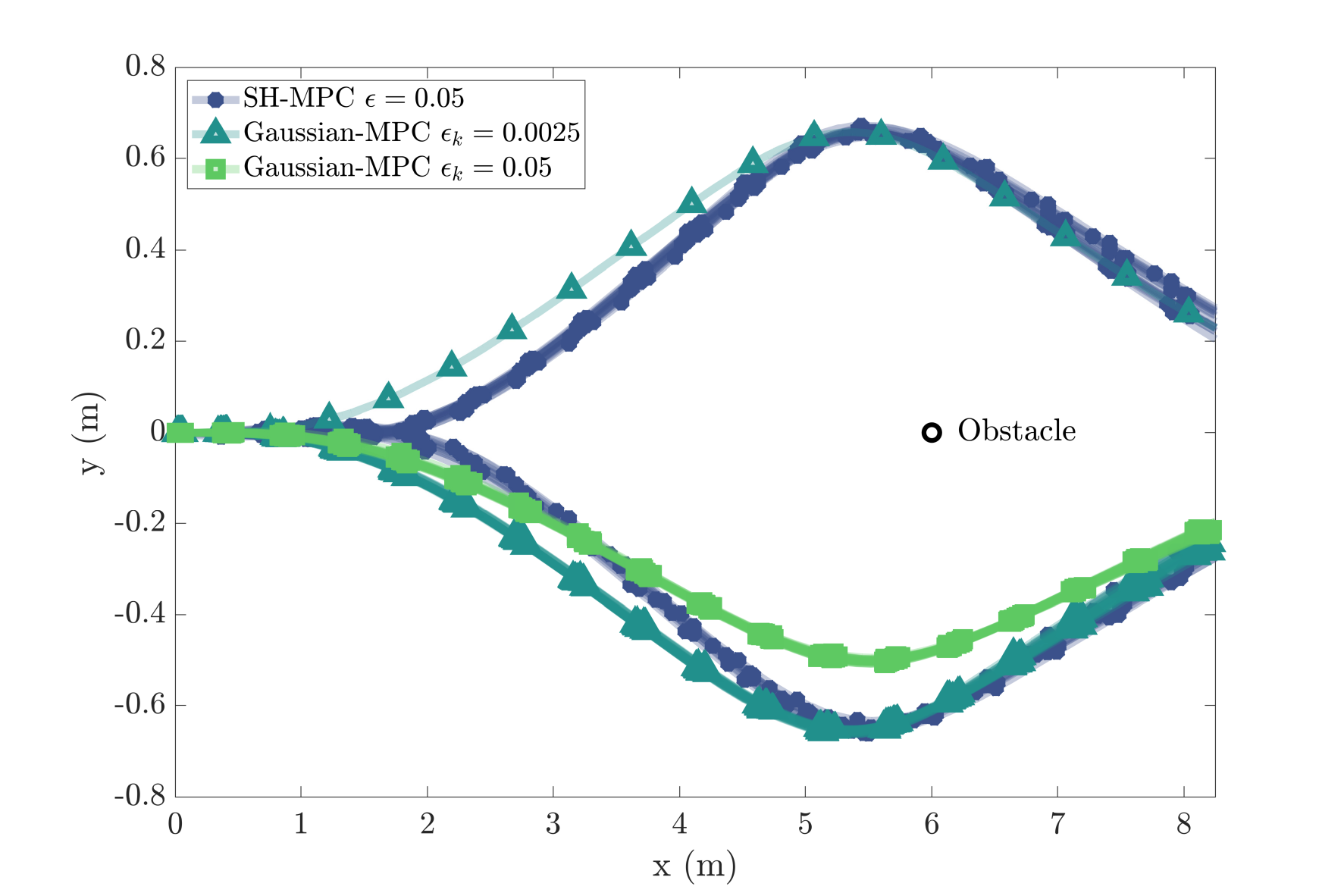}
        \caption{Trajectories around a static obstacle where motion predictions follow a Gaussian distribution.}
     \label{fig:trajectories_zero}
\end{figure}

The trajectories indicate that the Gaussian baselines need to react faster to the obstacle. SH-MPC and the $\epsilon_k = 0.0025$ Gaussian baseline pass at a similar distance to the obstacle, but under SH-MPC the \systems can steer later, possibly leading to faster trajectories. The Gaussian baseline at $\epsilon_k = 0.05$ is faster, but does not meet the safety specification, as we will show in the following simulations.%

\subsection{Mobile \Systems Simulations - Gaussian} \label{sec:results_gaussian}
We consider an environment with multiple dynamic pedestrians following dynamics similar to the previous case,
\begin{equation}
    \bm{\delta}_{k + 1} = \bm{\delta}_k + (\b{v} + \bm{\delta}_{w, k}) dt, \quad \bm{\delta}_{w,k} \sim \mathcal{N}(\b{0}, \bm{\Sigma}_{w}), \label{eq:ped_dynamics2}
\end{equation}
but where $\b{v}\in\mathbb{R}^2$ describes a constant velocity. We validate the actual CP of the motion plan offline after the experiments for all methods through Monte Carlo sampling in which the dynamics in Eq.~\eqref{eq:ped_dynamics2} are used to generate the samples. The CP is computed by dividing the number of samples where the \systems and obstacle discs overlap at any stage by the total number of samples (set to $10^5$). The marginal CP (CP$_k$) is computed without taking prior collisions into account. We validate in a scenario with 4 and 8 pedestrians, respectively. Fig.~\ref{fig:gaussian} depicts snapshots of the simulations with 4 pedestrians.

\begin{figure*}
    \centering
    \begin{subfigure}[b]{0.30\textwidth}
         \centering
         \includegraphics[width=\textwidth]{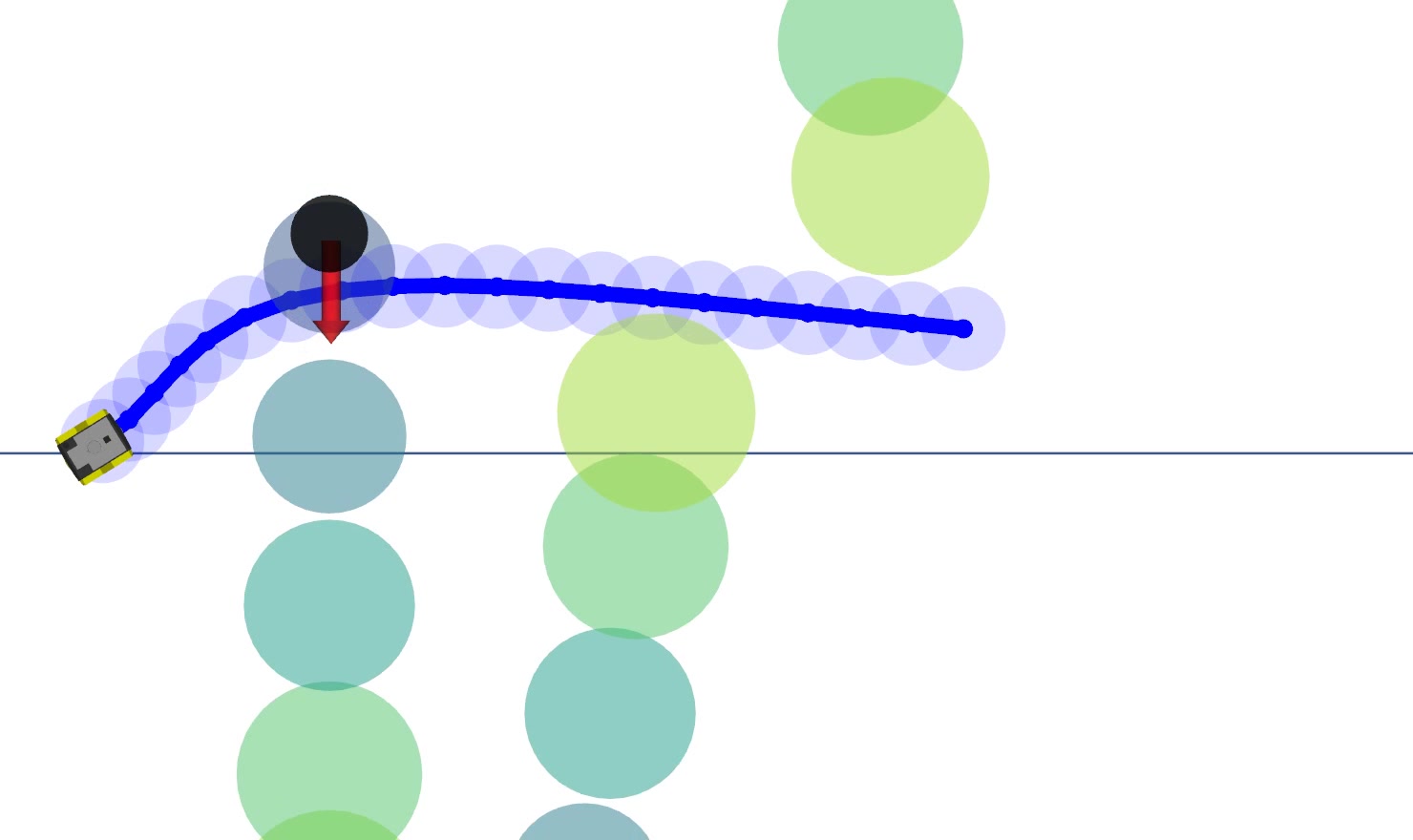}
            \caption{Gaussian baseline at $\epsilon_k=0.0025$.}
         \label{fig:gaussian_ellipsoid}
     \end{subfigure}
    \begin{subfigure}[b]{0.30\textwidth}
    \centering
         \includegraphics[width=\textwidth]{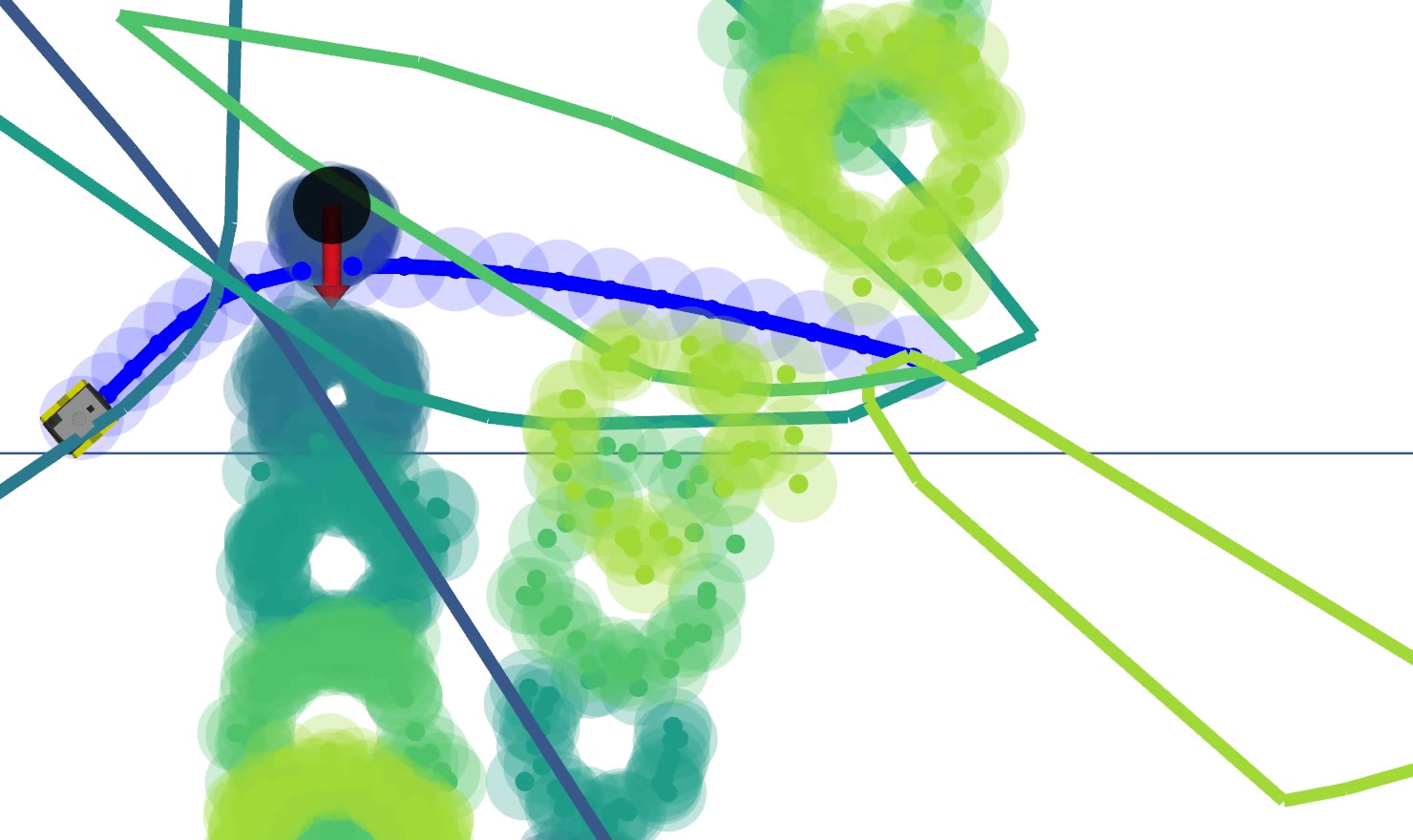}
            \caption{S-MPCC at $\epsilon_k=0.0025$.}
         \label{fig:gaussian_smpcc}
     \end{subfigure}
     \hspace{1pt}
     \begin{subfigure}[b]{0.30\textwidth}
         \centering
         \includegraphics[width=\textwidth]{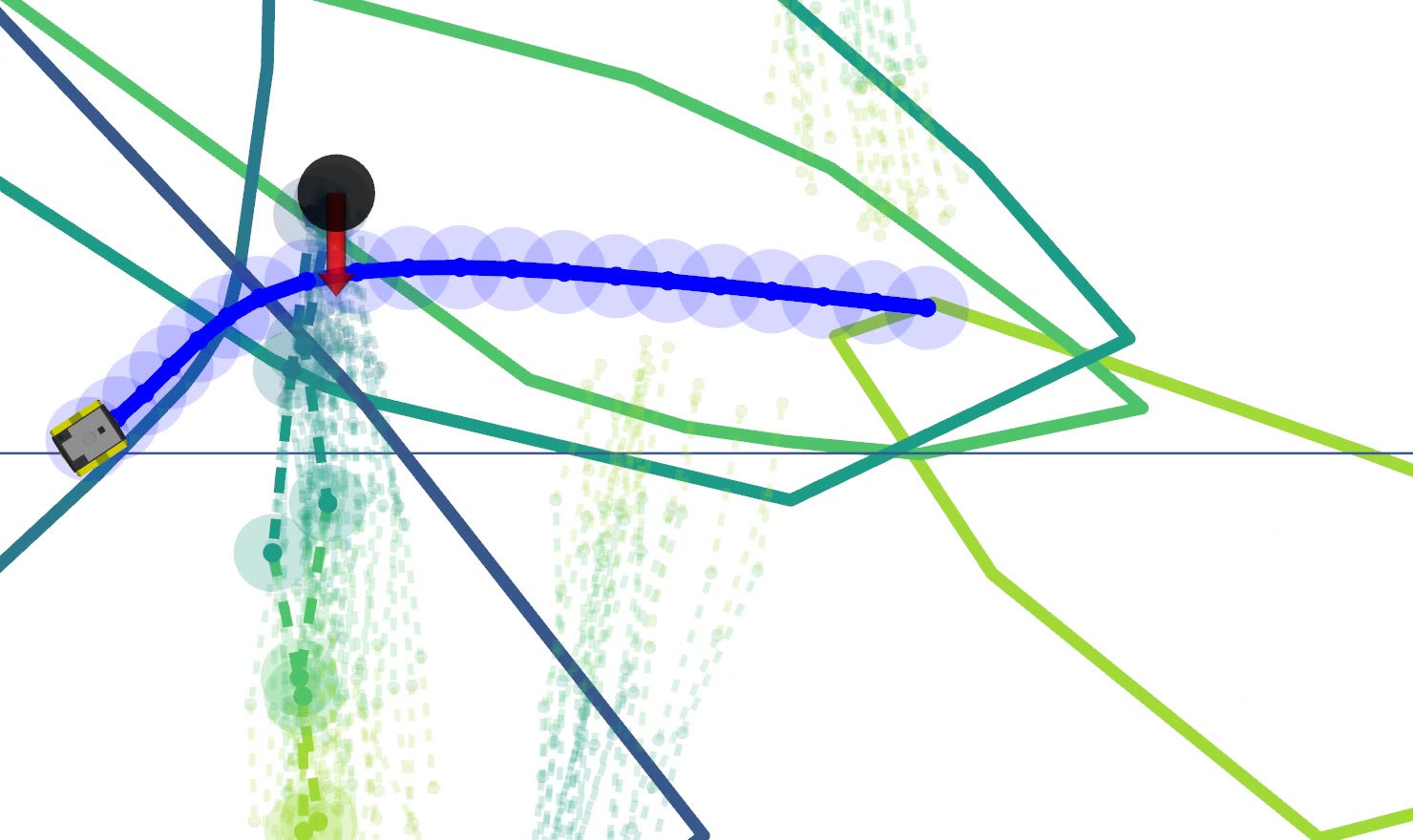}
            \caption{SH-MPC at $\epsilon=0.05$.}
         \label{fig:gaussian_shmpc}
     \end{subfigure}
    \caption{Snapshot of the simulation under Gaussian pedestrian motion. \h{Pedestrians are depicted by black circles with direction of motion highlighted in red. The robot plan and its collision area is drawn in blue. For all methods visualizations are shown in blue to green colors for stages $0, 5, 10, 15$ and $19$, respectively. (a) Circles show the level sets of the Gaussian distribution at the specified $\epsilon$. (b) Sampled pedestrian positions (excluding pruned samples in the center) are drawn as points with their collision area, borders of the safe polytopes are drawn as colored lines. (c) Similar to (b) but where sampled trajectories are drawn as dashed lines and support constraints are highlighted.}}
    \label{fig:gaussian}
\end{figure*}%
\begin{table*}
    \centering
    \caption{Statistical results \ho{over $100$ experiments} of the marginal \ho{CP (``CP$_k$'')} and trajectory CP \ho{(``CP'')}, the task duration, traveled distance, minimum distance to the pedestrians and computation times for \ho{the unimodal simulation with $4$ pedestrians}. \ho{For the CPs we report the maximum observed over all experiments and compare it to the specified bound (a dash indicates that no bound is specified on the particular CP).} Other results are reported as ``average (standard deviation)'' unless stated otherwise. \ho{Methods are grouped by their safety guarantee.}}
\begin{tabular}{|l|c@{\hspace{1pt}}c@{\hspace{1pt}}l@{\hspace{2pt}}c|c@{\hspace{1pt}}c@{\hspace{1pt}}l@{\hspace{2pt}}c|c|c|c|c|}
\hline \textbf{Method} & \multicolumn{4}{c|}{\textbf{Max CP$_k$/Spec. (\%)}} & \multicolumn{4}{c|}{\textbf{Max CP/Spec. (\%)}} & \textbf{Dur. [s]} &\textbf{Trav. [m]} &\textbf{Min Dist. [m]} &\textbf{Runtime (Max) [ms]}\\\hline
S-MPCC ($\epsilon_k=0.05$) & 0.0043 &/ &0.0500 &(9) & 0.0082 &/ &- &(-) & 5.56 (0.36) & 8.84 (0.08) & 0.27 (0.06) & 19 ~(98)\\\hline
Gaussian ($\epsilon_k=0.05$) & 0.0455 &/ &0.0500 &(91) & 0.1271 &/ &- &(-) & 5.27 (0.10) & 8.82 (0.05) & 0.13 (0.05) & 14 ~(64)\\
\hhline{|=============|}
S-MPCC ($\epsilon_k=0.0025$) & 0.0003 &/ &0.0025 &(10) & 0.0004 &/ &- &(-) & 6.39 (0.09) & 9.00 (0.05) & 0.39 (0.06) & 20 ~(77)\\\hline
Gaussian ($\epsilon_k=0.0025$) & 0.0025 &/ &0.0025 &(99) & 0.0069 &/ &- &(-) & \textbf{5.33} (0.08) & \textbf{8.87} (0.04) & 0.26 (0.06) & \textbf{12} (111)\\\hline
\hh{SH-MPC ($\epsilon = 0.05$)} & \hh{0.0070} &/ &- &(-) & \hh{0.0113} &/ &\hh{0.0500} &\hh{(23)} & \hh{5.64 (0.40) }& \hh{9.00 (0.10) }& \hh{0.28 (0.05) }& \hh{23 (61)}\\\hline
\end{tabular}
    \label{tab:gaussian}
\end{table*}%
\begin{table*}
    \centering
    \caption{\h{Results for the unimodal simulation with $8$ pedestrians. Displayed results follow the notation in Table \ref{tab:gaussian}.}}%
\begin{tabular}{|l|c@{\hspace{1pt}}c@{\hspace{1pt}}l@{\hspace{2pt}}c|c@{\hspace{1pt}}c@{\hspace{1pt}}l@{\hspace{2pt}}c|c|c|c|c|}
\hline \textbf{Method} & \multicolumn{4}{c|}{\textbf{Max CP$_k$/Spec. (\%)}} & \multicolumn{4}{c|}{\textbf{Max CP/Spec. (\%)}} & \textbf{Dur. [s]} &\textbf{Trav. [m]} &\textbf{Min Dist. [m]} &\textbf{Runtime (Max) [ms]}\\\hline
Gaussian ($\epsilon_k=0.05$) & \exceeds{0.0865} &/ &0.0500 &(173) & \exceeds{0.2309} &/ &- &(-) & 11.76 (3.07) &  18.53 (0.43) & 0.14 (0.05) &14 (73)\\
\hhline{|=============|}
Gaussian ($\epsilon_k=0.0025$) & \exceeds{0.0041} &/ &0.0025 &(162) & 0.0140 &/ &- &(-) & 16.11 (1.58)  & 19.11 (0.43)& 0.30 (0.07) & 13 (58)\\
\hhline{|=============|}
Gaussian ($\epsilon_k=0.0003125$) & 0.0004 &/ &0.0025 &(16) & 0.0016 &/ &- &(-) & 18.25 (3.31)  & \textbf{19.24} (0.42)& 0.37 (0.07) & \textbf{14} (59)\\\hline
\hh{SH-MPC ($\epsilon = 0.05$)} & \hh{0.0047} &/ &- &(-) & \hh{0.0128} &/ &\hh{0.0500} &\hh{(26)} & \hh{\textbf{15.63} (1.73) }& \hh{19.32 (0.22) }& \hh{0.34 (0.05) }& \hh{30 (79)}\\\hline
\end{tabular}
    \label{tab:gaussian_crowded}
\end{table*}%

The results with 4 pedestrians are summarized in Table~\ref{tab:gaussian}. The Gaussian method bounds the marginal CP (CP$_k$) accurately. However when $\epsilon_k = 0.05$, the CP of the plan is $0.1271$ which exceeds $0.05$. The trajectories are safe under $\epsilon_k = 0.0025$, but are conservative with a maximum overall CP of $0.0069$. In contrast, SH-MPC attains a higher, but still safe maximum CP of $\hh{0.0113}$. We do not observe a significant difference in the performance metrics of this simulation, the Gaussian method at $0.0025$ is slightly faster than SH-MPC but travels more distance. The repeated computations of the controller are likely responsible for the observed similarity in performance. Because marginal methods are accurate over the short term, they become less conservative close to obstacles. It is worth noting that while the Gaussian method uses information of the distribution, SH-MPC and S-MPCC are only using samples of the distribution. S-MPCC is more conservative than the other two methods as a result of its conservative support estimate.

For the case of 8 pedestrians, results are summarized in Table~\ref{tab:gaussian_crowded}. In these simulations, the marginal CP of the Gaussian method exceeds the specification. Fig.~\ref{fig:marginal_obstacle} illustrates on a 1-dimensional example that this is a result of the marginalization per obstacle. In this simulation we therefore also compare against the Gaussian baseline at $\epsilon_k =\frac{0.0025}{8}= 0.0003125$, which attains the same safety guarantees as SH-MPC. The results indicate that SH-MPC is able to move through this environment significantly faster than the baseline with the same safety guarantees. Our method shows no significant change in its overall CP compared to the previous simulation as it considers the joint distribution, while the risk of the baselines has more than doubled. Additionally, the safe baseline with $\epsilon_k = 0.0003125$ is excessively conservative with a CP of $0.0016$ versus \hh{$\h{0.0128}$} of SH-MPC.
\begin{figure}[t]
    \centering
    \includegraphics[width=0.38\textwidth]{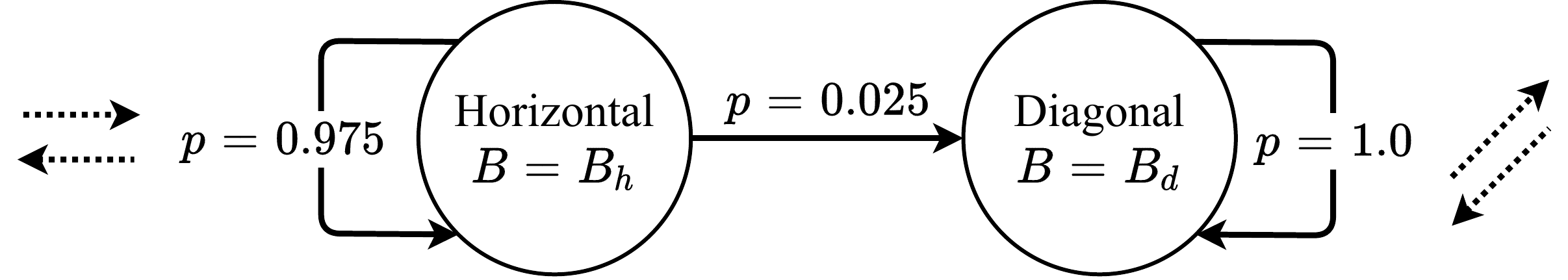}
    \caption{Markov-chain representation of crossing pedestrian motion.}
    \label{fig:markov_chain}
\end{figure}%
\subsection{Mobile \Systems Simulations - Gaussian Mixture Model}\label{sec:results_gmm}
We now modify the distribution to incorporate a probability that the pedestrians will cross. We encode this scenario with a Markov Chain (see Fig.~\ref{fig:markov_chain}) that changes pedestrian movement from horizontal to diagonal in addition to the Gaussian process noise of the previous simulations. We define the pedestrian dynamics as
\begin{equation}
    \b{p}_{k + 1} = \b{p}_k + (B\b{v} + \bm{\delta}_{w, k}) dt, \quad \bm{\delta}_{w, k} \sim \mathcal{N}(\b{0}, \bm{\Sigma}_{w, k}),
\end{equation}
where $B$ is either $B_h = \begin{bmatrix} 1 & 0 \end{bmatrix}^T$ or $B_d = \begin{bmatrix} \frac{1}{\sqrt{2}} & \frac{1}{\sqrt{2}} \end{bmatrix}^T$ depending on the state of the Markov Chain. The uncertainties associated with this motion can be modeled as a Gaussian Mixture Model (GMM), where each possible state transition in the Markov Chain leads to a separate mode with an associated probability (in total $21$ modes). For example, the mode where the pedestrian crosses after two steps occurs with probability $p_2 = (1 - 0.025)0.025$. The setup of these simulations is visualized in Fig.~\ref{fig:binomial}. We apply the Gaussian baseline to the GMM distribution by formulating the constraints at a risk of $\epsilon_k$ for all modes, which can lead to conservatism when multiple modes influence the plan. \hh{Since previous simulations showed that S-MPCC performs worse than the other two methods, we do not include it in this comparison.} We validate in this environment with $8$ pedestrians.

\begin{figure*}
    \centering
    \begin{subfigure}[b]{0.3\textwidth}
         \centering
         \includegraphics[width=\textwidth]{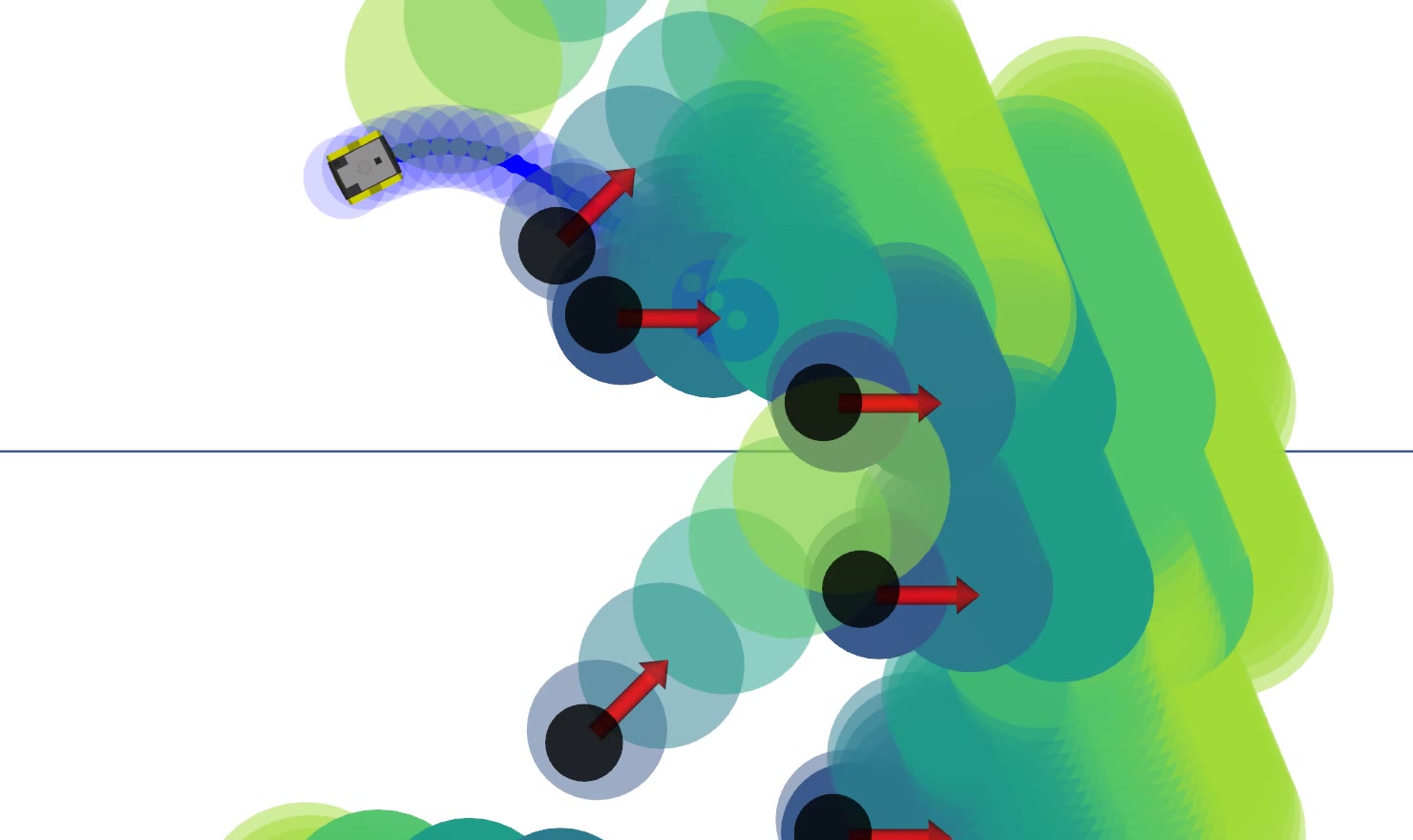}
            \caption{Gaussian baseline at $\epsilon_k=0.0003125$.}
         \label{fig:binomial_ellipsoid}
     \end{subfigure}
     \hspace{1pt}
     \begin{subfigure}[b]{0.3\textwidth}
         \centering
         \includegraphics[width=\textwidth]{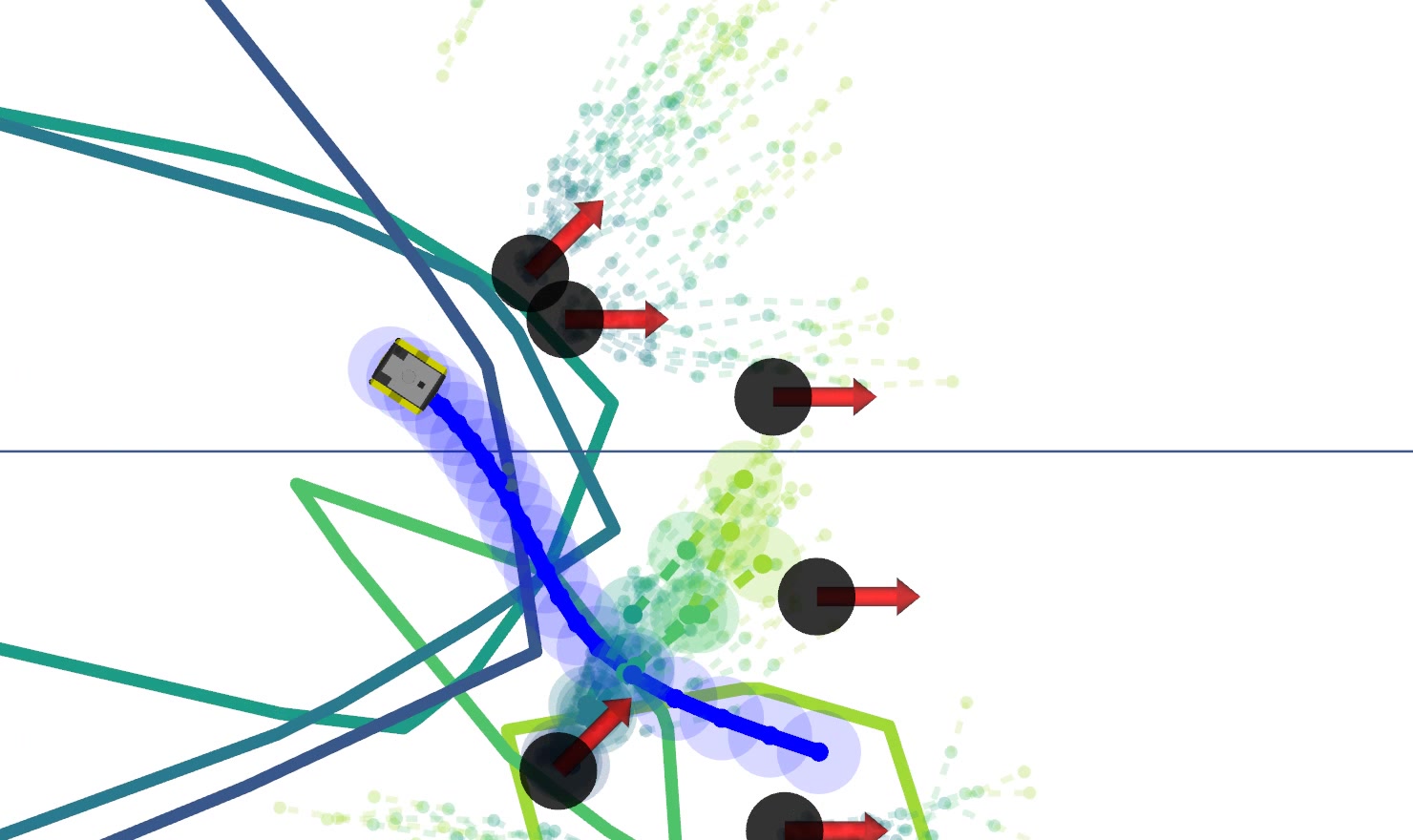}
            \caption{SH-MPC at $\epsilon=0.05$.}
         \label{fig:binomial_shmpc}
     \end{subfigure}
    \caption{Snapshot of the simulation where pedestrian motion follows a GMM using the same visualizations as in Fig.~\ref{fig:gaussian}. \h{\mbox{(a) Level} sets of each mode are visualized. (b) Sampled trajectories consist of crossing and non-crossing behavior.}}
    \label{fig:binomial}
\end{figure*}%

\begin{table*}
    \centering
    \caption{\h{Results for the multimodal simulation with $8$ pedestrians. Displayed results follow the notation in Table \ref{tab:gaussian}, with minimum distance replaced by the number of collisions.}}
\begin{tabular}{|l|c@{\hspace{1pt}}c@{\hspace{1pt}}l@{\hspace{2pt}}c|c@{\hspace{1pt}}c@{\hspace{1pt}}l@{\hspace{2pt}}c|c|c|c|c|}
\hline \textbf{Method} & \multicolumn{4}{c|}{\textbf{Max CP$_k$/Spec. (\%)}} & \multicolumn{4}{c|}{\textbf{Max CP/Spec. (\%)}} & \textbf{Dur. [s]} &\textbf{Trav. [m]} &\textbf{Collisions} &\textbf{Runtime (Max) [ms]}\\\hline
Gaussian ($\epsilon_k=0.05$) & 0.1042 &/ &0.0500 &(208) & 0.3346 &/ &- &(-) & 16.98 (4.69) & 20.27 (1.17) & 8 & 144 (451)\\
\hhline{|=============|}
Gaussian ($\epsilon_k=0.0025$) & 0.0050 &/ &0.0025 &(198) & 0.0215 &/ &- &(-) & 17.97 (4.94) & 20.51 (1.09) & 6 & 150 (447)\\
\hhline{|=============|}
Gaussian ($\epsilon_k=0.0003125$) & 0.0006 &/ &0.0025 &(24) & 0.0028 &/ &- &(-) & 18.59 (5.32) & 20.66 (1.11) & \textbf{0} & 148 (409)\\\hline
\hh{SH-MPC ($\epsilon = 0.05$)} & \hh{0.0057} &/ &- &(-) & \hh{0.0126} &/ &\hh{0.0500} &\hh{(25)} & \hh{\textbf{18.02} (3.66) }& \hh{\textbf{18.74} (0.23) }& 3 & \hh{\textbf{29} (71)}\\\hline
\end{tabular}
    \label{tab:binomial_crowded}
\end{table*}%

Results are summarized in Table~\ref{tab:binomial_crowded}. We again validate against the baseline with $\epsilon_k = 0.0003125$. \hbox{SH-MPC} outperforms the baselines on almost all metrics. We also note that the computation times of the Gaussian method are excessive due to the many modes to be considered, while the computation times of SH-MPC are unaffected. Some collisions occur for \h{the higher risk} methods in this environment. Collisions are observed when a pedestrian starts crossing towards the \system. In this case the predictions change between two \ho{instances of the control problem} and make the optimization infeasible, freezing the robot in place. %

\begin{figure}[t!]
    \centering
    \begin{subfigure}[t]{0.12\textwidth}
    \centering
         \scalebox{1}[1]{\includegraphics[width=\textwidth]{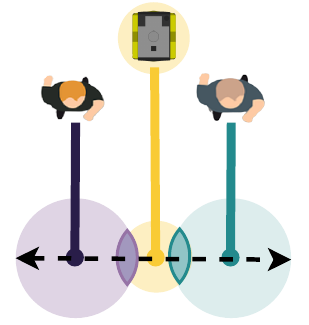}}
            \caption{\h{Example case}}%
         \label{fig:marginal_a}
     \end{subfigure}
     \hspace{2pt}
     \begin{subfigure}[t]{0.31\textwidth}
         \centering
         \includegraphics[width=\textwidth, trim={0.5cm 0cm 1.5cm 0cm}, clip]{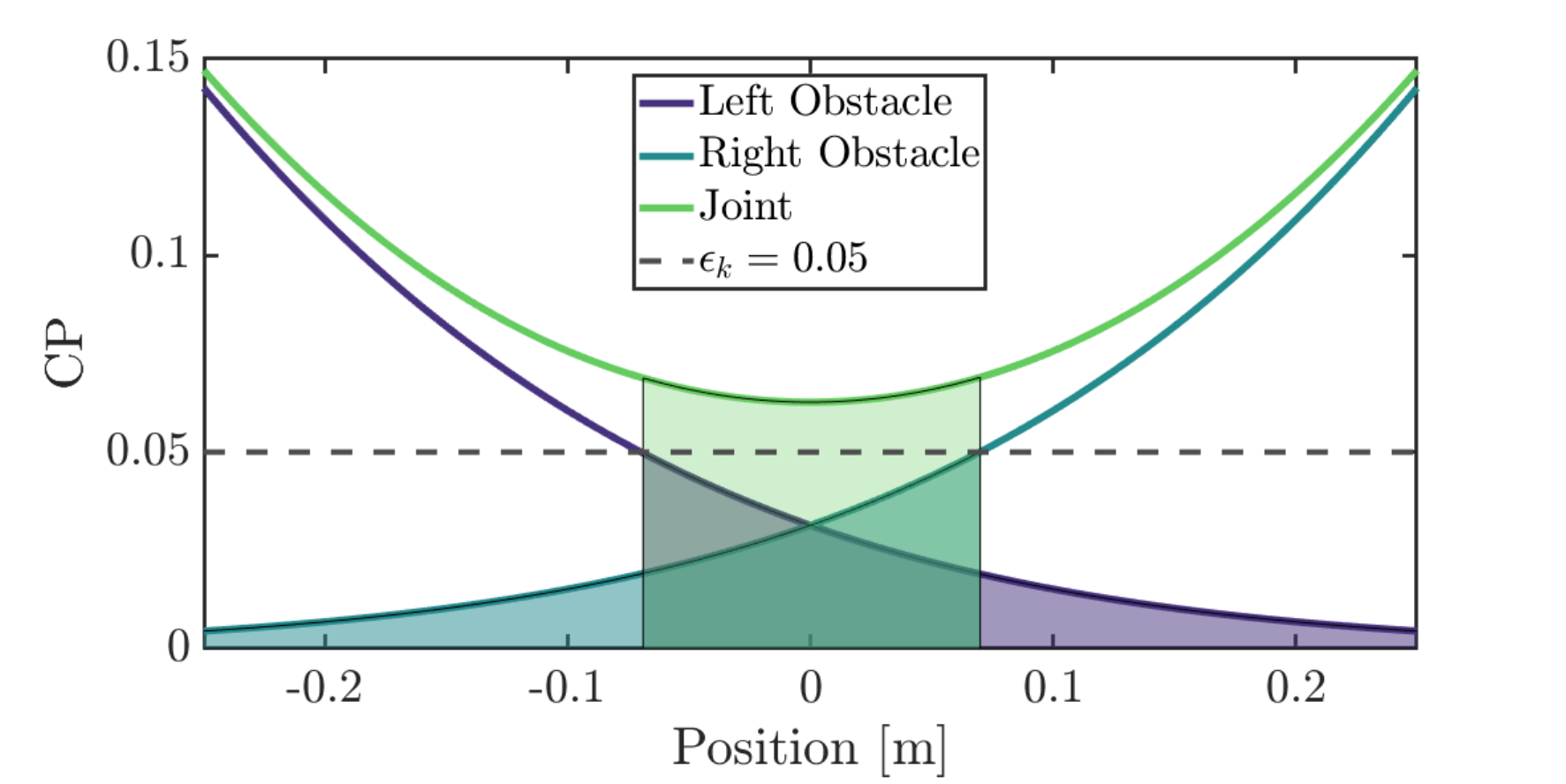}
        \caption{Joint and marginalized risk in the direction of the black dashed arrow in (a).}%
         \label{fig:marginal_b}
     \end{subfigure}
    \caption{A 1D illustration of the case where two obstacles constrain the \system. Even though the marginal probability of collision for each obstacle is less than $\epsilon_k$ (shaded tails) in the center region, the joint probability of collision in the feasible region (green shaded area) is larger than $\epsilon_k$.}
    \label{fig:marginal_obstacle}
\end{figure}

\subsection{\ho{Empirical Analysis and Sensitivity}}%
Fig~\ref{fig:risk_figure} visualizes the empirical risk distribution of the multimodal 8 pedestrian simulation (Sec.~\ref{sec:results_gmm}) with respect to the support of the SP (Eq.~\ref{eq:scenario_program}). This distribution empirically confirms that higher support leads to higher risk on average as well as longer tails. %

\ho{We validate the sensitivity of SH-MPC to varying risk \mbox{specifications} ($\epsilon$) and horizon lengths ($N$) in the multimodal 8 pedestrian simulations. Results are depicted in Fig.~\ref{fig:sensitivity}. Fig.~\ref{fig:sensitivity_epsa} shows that reducing the specified risk results in longer task durations. It also shows that the approach becomes more conservative for lower risk specifications. In accordance with our approach, Fig~\ref{fig:sensitivity_N} indicates that changing the horizon length does not significantly affect the CP. Instead, the trajectories become more cautious when the same risk must be guaranteed over a longer duration. We observe in Fig.~\ref{fig:sensitivity_eps_runtime} and \ref{fig:sensitivity_N_runtime} that the computation times are mainly increasing when the specified risk is reduced. This is due to its relation with the sample size.}

\begin{figure}[t]
    \centering
        \includegraphics[width=0.39\textwidth,trim={0.5cm, 0.2cm, 1.5cm, 0.6cm}, clip]{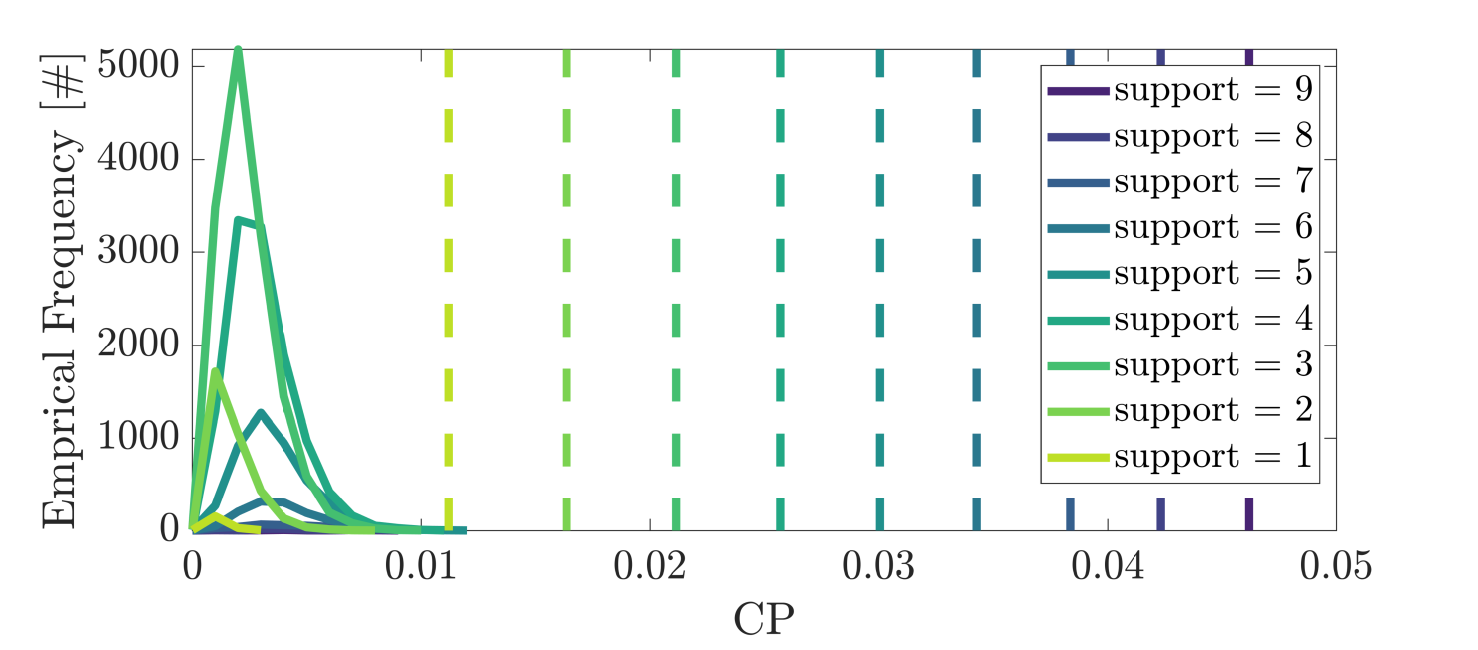}%
    \caption{The empirical distribution of the trajectory CP obtained by SH-MPC for each support value under pedestrian motion distributed as a Gaussian Mixture Model. Dashed vertical lines denote the theoretical CP bound for each support.}
    \label{fig:risk_figure}
\end{figure}
\begin{figure}[t]
    \centering
    \begin{subfigure}[b]{0.23\textwidth}
         \centering
         \includegraphics[width=\textwidth]{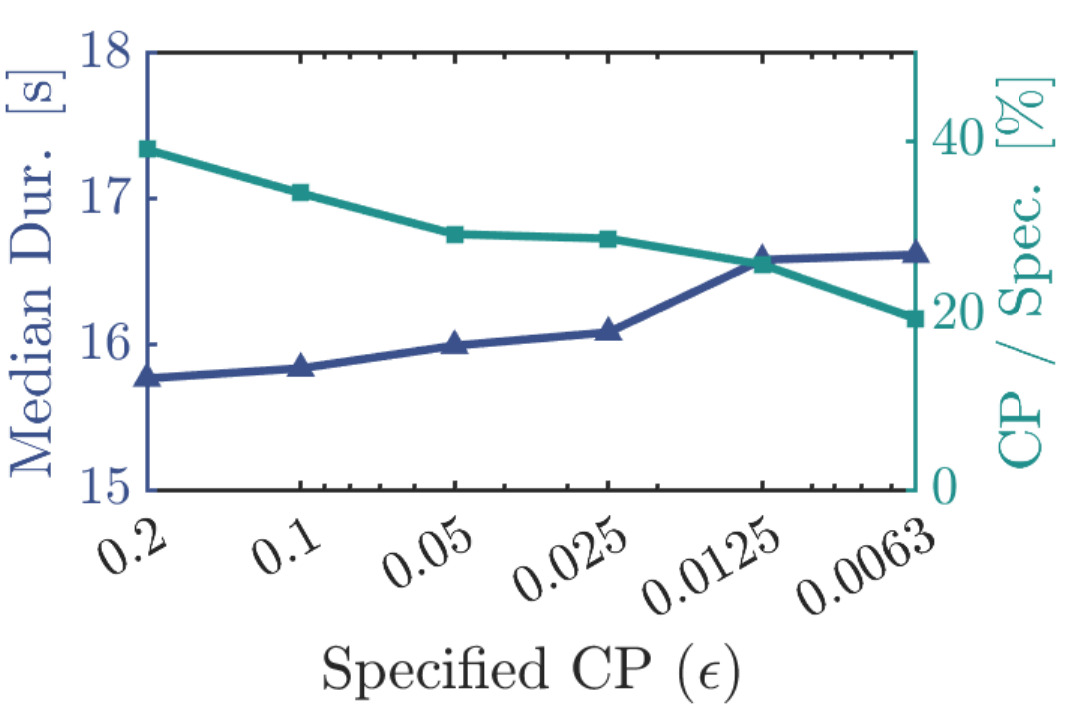}
            \caption{}
         \label{fig:sensitivity_epsa}
     \end{subfigure}
         \begin{subfigure}[b]{0.23\textwidth}
         \centering
         \includegraphics[width=\textwidth]{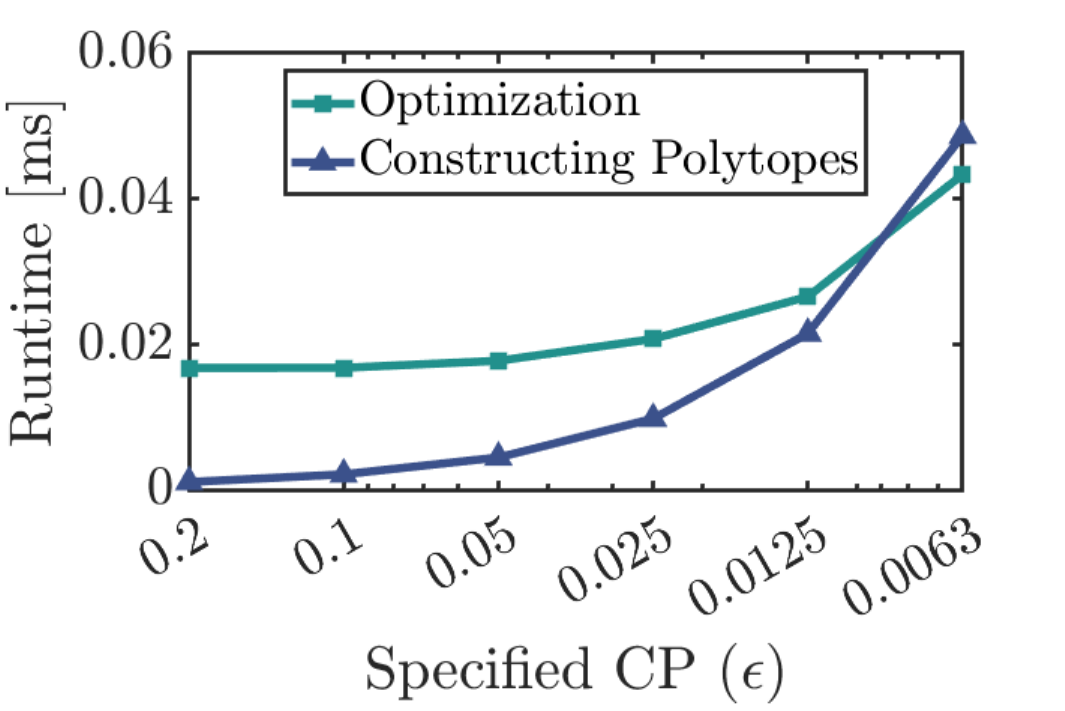}
            \caption{}
         \label{fig:sensitivity_eps_runtime}
     \end{subfigure}
     \newline
     \begin{subfigure}[b]{0.23\textwidth}
         \centering
         \includegraphics[width=\textwidth]{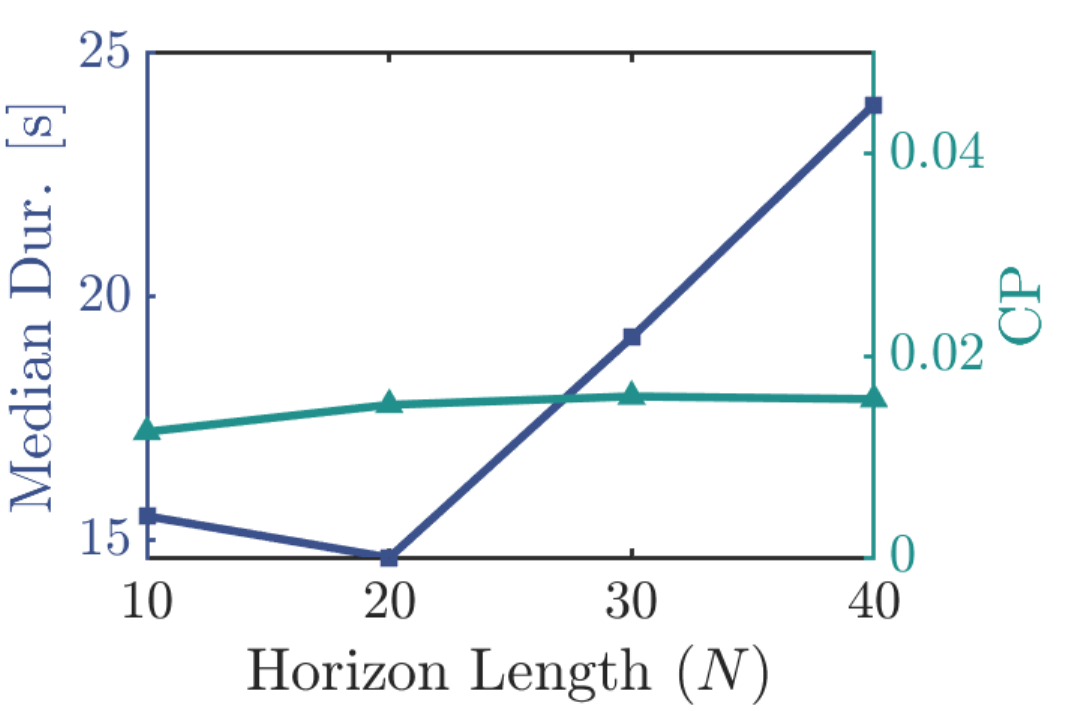}
            \caption{}
         \label{fig:sensitivity_N}
     \end{subfigure}
     \begin{subfigure}[b]{0.23\textwidth}
         \centering
         \includegraphics[width=\textwidth]{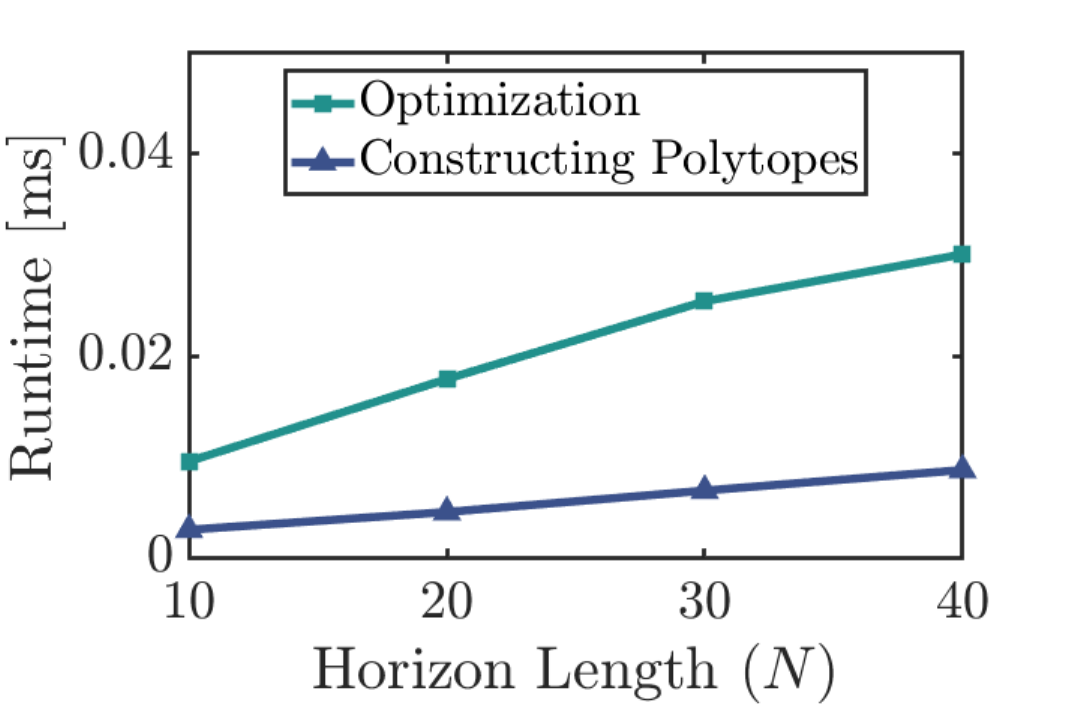}
        \caption{}
        \label{fig:sensitivity_N_runtime}
     \end{subfigure}
    \caption{\ho{Sensitivities of the empirical CP and runtimes with respect to the specified CP ($\epsilon$) and the horizon length ($N$), evaluated over $10$ experiments in the setting of Sec.~\ref{sec:results_gmm}. Task duration is denoted ``Dur.''. With ``CP'' we denote the mean over the experiments of the maximum trajectory CP. The runtimes are separated in solving the optimization and computing free-space polytopes \eqref{eq:polytope} from the scenarios.}}
    \label{fig:sensitivity}
\end{figure}

\subsection{Autonomous Navigation in an Urban Environment}
\ho{SH-MPC can be applied to \h{different robot morphologies and scenarios}. As demonstration, we deploy our approach on a simulated self-driving vehicle in Carla simulator~\cite{dosovitskiy_carla_2017}.} The dynamics are modeled with a second order bicycle model~\cite{kong_kinematic_2015} and the collision region consists of $3$ discs. We construct a collision-free polytope for each of the discs. The pedestrians are programmed to follow the same dynamics as in Sec.~\ref{sec:results_gmm}, i.e., a GMM modeling crossing behavior. \h{We do not model the interaction between the vehicle and the pedestrians}. The control frequency is $10$~Hz. We measured the computation time to be \h{$88$}~ms on average and \h{$135$}~ms maximum. Fig.~\ref{fig:carla} visualizes snapshots of the simulations. \ho{A video of the simulations is provided with this paper~\cite{o_de_groot_scenario-based_2022}.} \h{In Case~A (see Fig.~\ref{fig:carla-a}), the planner keeps sufficient distance to let the pedestrians cross, while driving as close to the path as is safe. In Case~B (see Fig.~\ref{fig:carla-b}) the vehicle passes behind the pedestrians while keeping distance to the pedestrian that is not crossing.}%

\begin{figure}[t]
    \centering
    \begin{subfigure}[b]{0.40\textwidth}
        \centering
        \includegraphics[width=\textwidth]{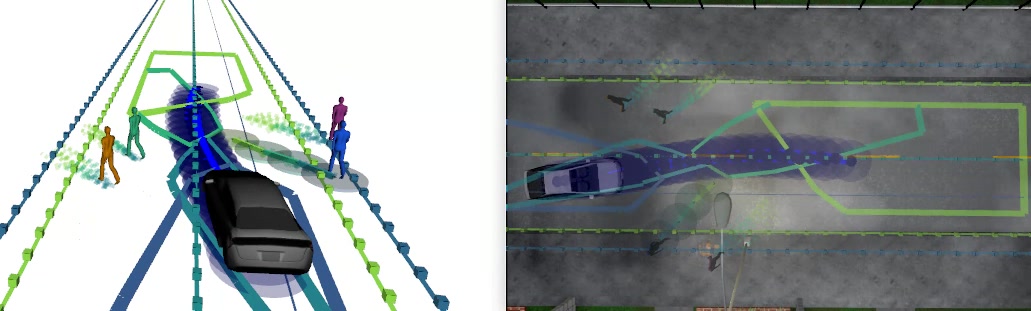}
        \caption{\h{A snapshot of the proposed approach in the Carla simulator.}}
        \label{fig:t1}
     \end{subfigure}
    \newline
     \begin{subfigure}[b]{0.18\textwidth}
        \centering
        \includegraphics[trim={1.6cm 0.8cm 4cm 1.5cm},clip,width=\textwidth]{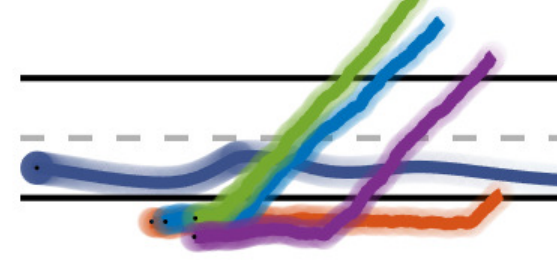}
        \caption{\h{Case A}}
        \label{fig:carla-a}
     \end{subfigure}
     \hspace{2pt}
     \begin{subfigure}[b]{0.18\textwidth}
        \centering
        \includegraphics[trim={1.6cm 0.8cm 4cm 1.5cm},clip,width=\textwidth]{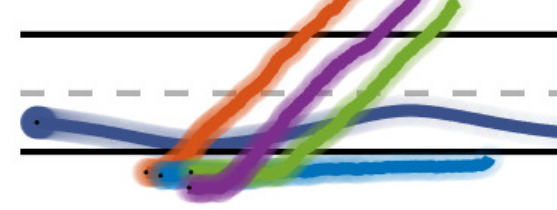}
        \caption{\h{Case B}}
        \label{fig:carla-b}
     \end{subfigure}
    \caption{\h{(a) Snapshot from Carla simulations with visualizations as in Fig.~\ref{fig:binomial_shmpc} for the frontal vehicle disc in stages $0, 6, 12$ and $19$. Plots (b) and (c) show the observed trajectories of the vehicle (dark blue) and pedestrians (other colors) in two cases.}}
    \label{fig:carla}
\end{figure}

\section{Discussion}
\ho{As expected from the theoretical analysis, our experiments showed that SH-MPC can consistently bound the CP of the overall trajectory. Where methods that impose marginal constraints need to decide between performance ($\epsilon_k = \epsilon$) and safety ($\epsilon_k = \frac{\epsilon}{NM}$), SH-MPC makes this \hbox{trade-off} more explicit by ensuring that the CP remains consistent under different operating conditions, such as with regards to number of obstacles, probability distributions, and the horizon length.}

\h{The gap between the risk guarantee and the obtained risk can be reduced, for example, by assuming some knowledge of the distribution or by running multiple scenario programs in parallel. Alternatively, the risk could be analyzed in continuous time (see for example~\cite{frey_collision_2020}) to reduce discretization errors.}

The proposed method is widely applicable. \h{It handles dynamic obstacles (e.g., cyclists, cars or non-cooperative robots) and controls systems with nonlinear dynamics. In addition, the joint distribution of the uncertainty can capture interactions of dynamic obstacles with other obstacles or the robot (e.g., to predict that pedestrians evade other pedestrians and the vehicle). It cannot yet account for interaction during planning, where the dynamics of the robot and pedestrians directly influence each other, as the probability measure $\mathbb{P}$ cannot depend on the optimization variables in scenario optimization.}

\h{The guarantees provided in this paper rely on an accurate model of the uncertainty, which may be challenging to obtain, for example in the case of human motion prediction. Nevertheless, the proposed method provides a planner that attains a desired level of risk with respect to the predicted probability distribution. The prediction model could also be replaced by recorded samples, in line with the typical scenario approach, to reduce modeling errors and provide formal guarantees with respect to the true motion of the obstacles.}

\ho{In terms of computational efficiency, we showed that our approach is online capable and scalable under typical operating conditions. For extreme low risk specifications (e.g., $\epsilon \leq 5\cdot 10^{-3}$), computational requirements may become excessive as a result of the increase in sample size. This can be addressed}\lf{, for example, by either} \ho{pruning the samples, given that only the extreme samples are of interest (see for example~\cite{de_groot_scenario-based_2021}), or by solving an approximate scenario optimization~\cite{sartipizadeh_approximate_2020}. In addition, most computations of SH-MPC are parallel linear computations (for each sample), which} \lf{potentially} \ho{leave room for further optimization, e.g., by delegating computations to a Graphical Processing Unit (GPU).}

\section{Conclusions}\label{sec:conclusion}
We presented a novel method for planning under uncertainty, Safe Horizon Model Predictive Control (SH-MPC), that bounds the collision probability of the planned trajectory over its duration and with respect to all obstacles. The method uses a scenario optimization formulation where samples of the involved uncertainty are used as constraints to limit the collision probability of the motion plan. The number of samples is the main indicator for the collision probability and under our approach could be computed before deploying the controller.

Our simulations, \h{with a mobile robot and an autonomous vehicle,} showed that SH-MPC better approximates the collision probability over \h{the duration of the motion plan} than existing methods that rely on the marginal probability of collision. The main baseline, that achieved tight evaluations of the risk for each time step, was shown to be conservative over the duration of its motion plan and there was significant variation in its overall risk when we varied the number of obstacles or their distribution. The overall risk of SH-MPC remained less conservative and more consistent between different environments and distributions, which resulted in faster trajectories when the environment was crowded. In addition, we showed excellent scaling of the computation time with respect to the number of obstacles and under varying distributions.

\section*{Acknowledgments}
The authors would like to thank Bruno Brito for assisting in the simulated experiments.

\newpage
\bibliographystyle{IEEEtran}
\bibliography{abrv, references_oscar}

\begin{IEEEbiography}[{\includegraphics[width=1in,height=1.25in,clip,keepaspectratio]{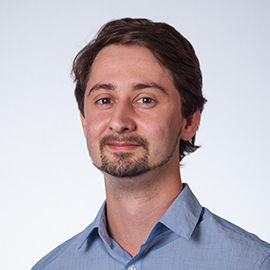}}]{Oscar de Groot} received both the B.Sc. degree in electrical engineering, in 2016, and the M.Sc. degree in systems \& control, in 2019, from the Delft University of Technology, Delft, The Netherlands. He is currently pursuing a Ph.D. in motion planning for autonomous vehicles in urban environments at the department of Cognitive Robotics at the Delft University of Technology. His research interests include probabilistic safe motion planning, scenario optimization, model predictive control and self-driving vehicles.
\end{IEEEbiography}
\begin{IEEEbiography}[{\includegraphics[width=1in,height=1.25in,clip,keepaspectratio]{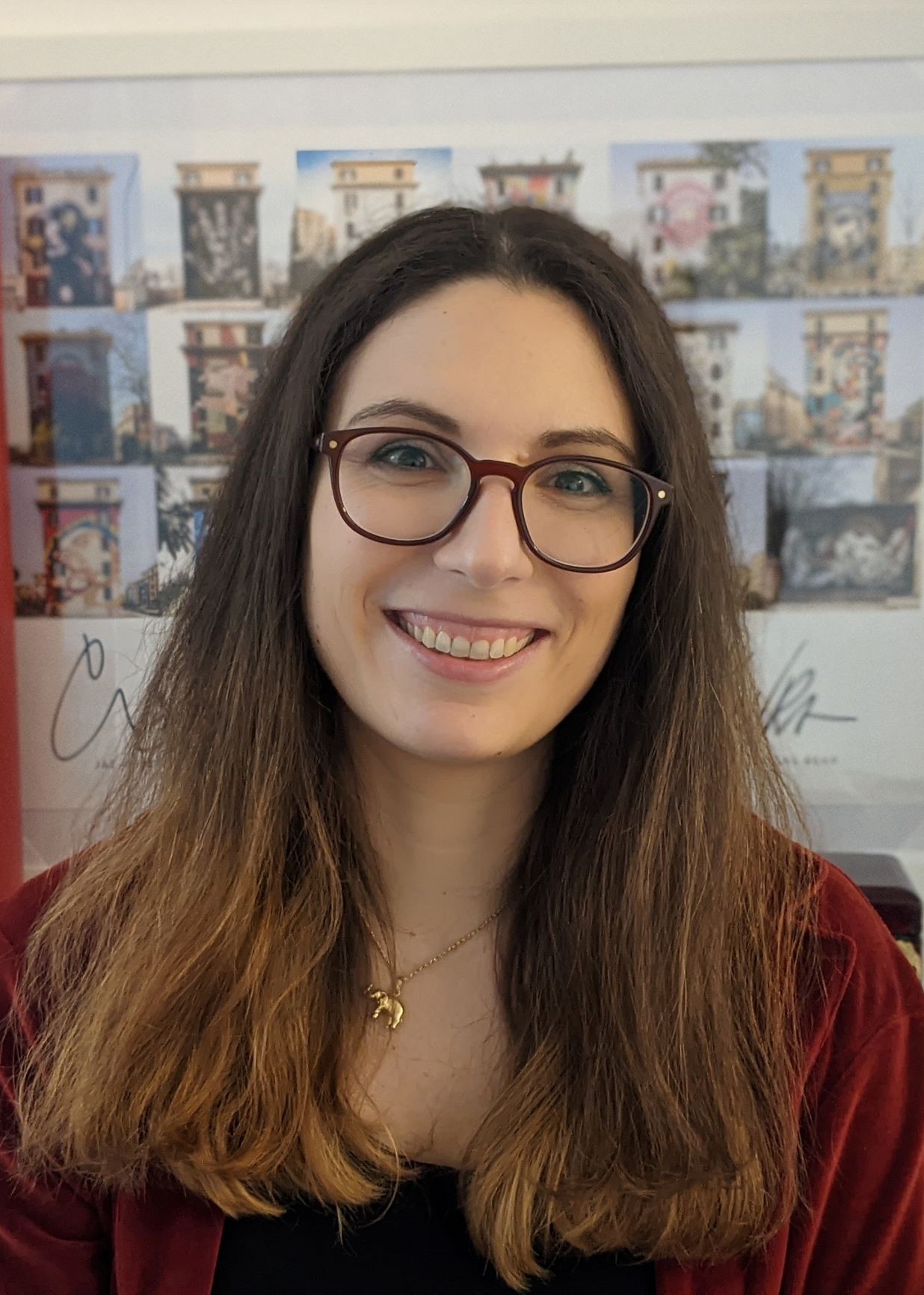}}]{Laura Ferranti} received her PhD from Delft University of Technology, Delft, The Netherlands, in
2017. She is currently an assistant professor in the Cognitive
Robotics (CoR) Department, Delft University of Technology, Delft, The
Netherlands. She is the recipient of an NWO Veni Grant from The
Netherlands Organisation for Scientific Research (2020), and of the Best Paper
Award in Multi-robot Systems at ICRA 2019.
Her research interests include optimization and optimal control,
model predictive control, reinforcement learning, embedded optimization-based
control with application in flight control, maritime transportation, robotics, and automotive.   
\end{IEEEbiography}
\begin{IEEEbiography}[{\includegraphics[width=1in,height=1.25in,clip,keepaspectratio]{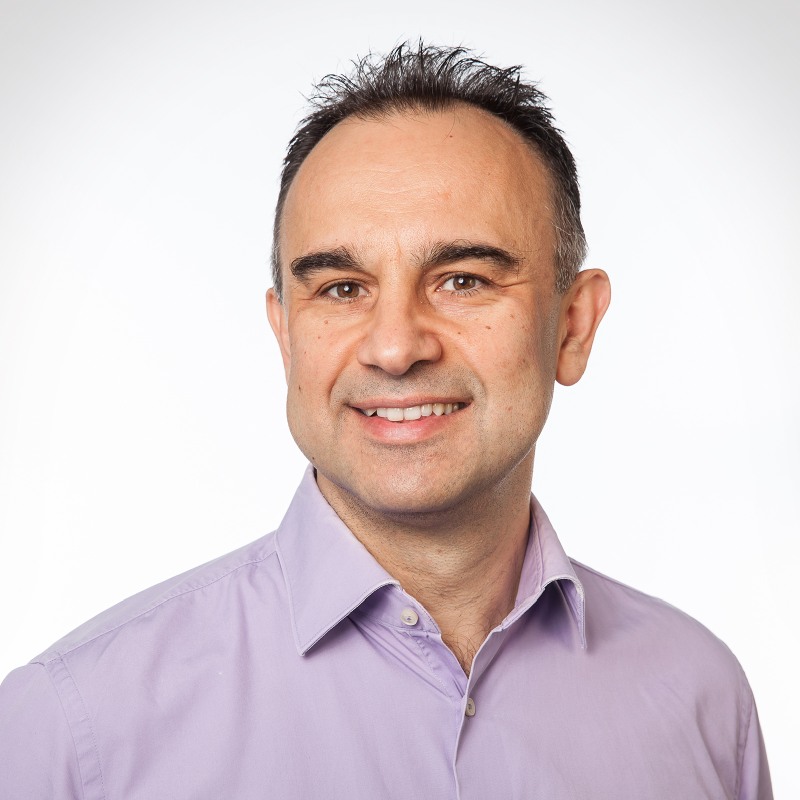}}]{Dariu Gavrila}
 received the Ph.D. degree in computer science from Univ. of Maryland at College Park, USA, in 1996. From 1997 until 2016, he was
with Daimler R\&D, Ulm, Germany, where he became a Distinguished Scientist. He led the vision-based pedestrian detection research, which was commercialized 2013-2014 in various Mercedes-Benz models. In 2016, he moved to TU Delft, where he since heads the Intelligent Vehicles group as a Full Professor. His current research deals with sensor-based detection of humans and analysis of behavior in the context of self-driving vehicles. He was awarded the Outstanding Application Award 2014 and the Outstanding Researcher Award 2019, both from the IEEE Intelligent Transportation Systems Society.
\end{IEEEbiography}
\begin{IEEEbiography}[{\includegraphics[width=1in,height=1.25in,clip,keepaspectratio]{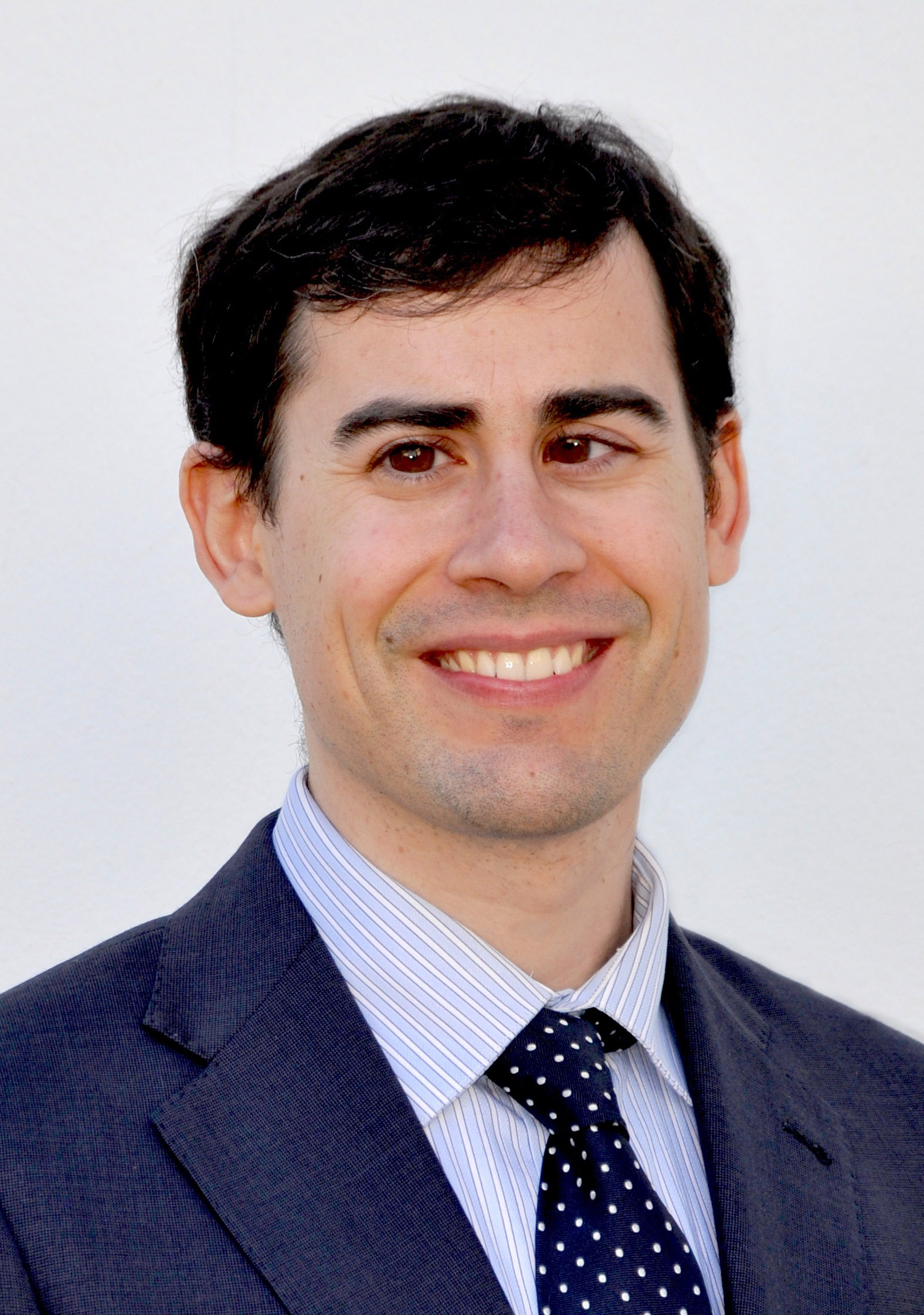}}]{Javier Alonso-Mora}
 is an Associate Professor at the Cognitive Robotics department of the Delft University of Technology, and a Principal Investigator at the Amsterdam Institute for Advanced Metropolitan Solutions (AMS Institute).
Before joining TU Delft, Dr. Alonso-Mora was a Postdoctoral Associate at the Massachusetts Institute of Technology (MIT). He received his Ph.D. degree in robotics from ETH Zurich,

His main research interest is in navigation, motion planning and control of autonomous mobile robots, with a special emphasis on multi-robot systems, on-demand transportation and robots that interact with other robots and humans in dynamic and uncertain environments. He is the recipient of an ERC Starting Grant (2021), the ICRA Best Paper Award on \ho{Multi-Robot} Systems (2019), an Amazon Research Award (2019) and a talent scheme VENI \ho{Grant} from the Netherlands Organisation for Scientific Research (2017).
\end{IEEEbiography}

\appendix
\section{\hh{Appendix}}\label{ap:shadow}
\hh{In practice, linearization of the collision avoidance constraints results in a support that consist of a small subset of the scenarios. To support this observation, we consider here how the constraint formulation impacts the support.}

\begin{definition}\label{def:shadow}
\hh{The \textit{shadow} of a scenario $\b{\delta}^{(i)}$ under the constraints $\b{\theta} \in \Theta_{\delta^{(i)}}$ is a region $S_{\b{\delta}^{(i)}} \subseteq \Delta$ such that if another scenario satisfies $\b{\delta} \in S_{\delta^{(i)}}$, then either $\b{\delta}^{(i)}$ or $\b{\delta}$ is redundant. That is,
\begin{equation*}
S_{\b{\delta}^{(i)}} = \{\b{\delta}\in\Delta \ | \ \Theta_{\b{\delta}^{(i)}} \subset \Theta_{\b{\delta}}\} \ \cup \ \{\b{\delta}\in\Delta \ | \ \Theta_{\b{\delta}} \subset \Theta_{\b{\delta}^{(i)}}\}.
\end{equation*}}
\end{definition}%
\hh{An example shadow is visualized in Fig.~\ref{fig:shadow_a} for the proposed constraints. Note that in Fig.~\ref{fig:shadow_b}, the shadow always occupies a non zero area. This is a result of the linearization and a set of box constraints on the position that limit the range in which constraints are considered. In this case we can prove the following.}
\begin{figure}[b]
    \centering
    \begin{subfigure}[t]{0.22\textwidth}
         \centering
         \includegraphics[width=\textwidth,trim={-4cm 0cm -4cm 0cm},clip]{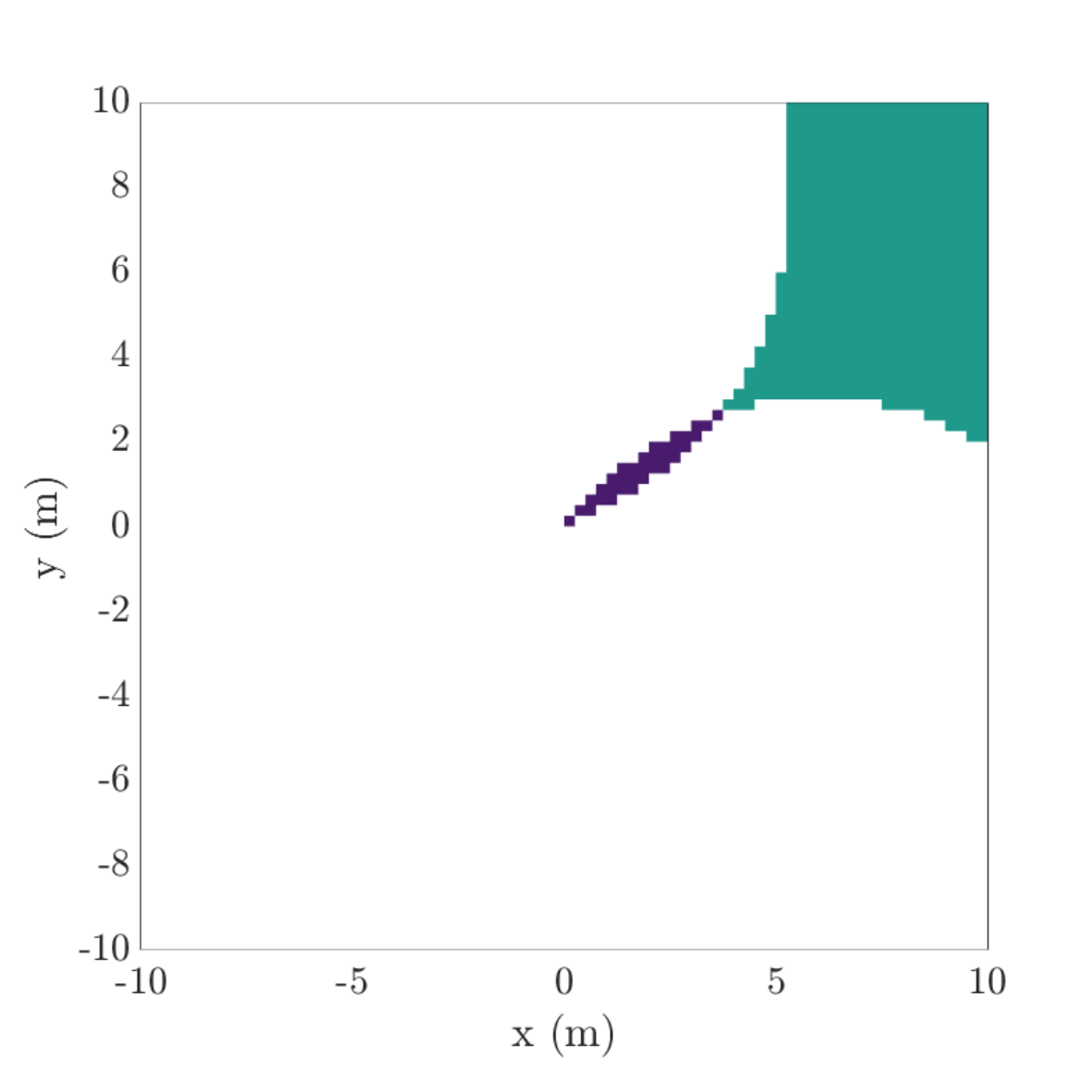}
        \caption{\hh{Shadow for a sample at $(2.5, 3.5)$ in blue (dominated by) and green (dominates).}}
         \label{fig:shadow_a}
     \end{subfigure}
     \hspace{1pt}
    \begin{subfigure}[t]{0.22\textwidth}
    \centering
         \includegraphics[width=\textwidth,trim={-4cm 0cm -4cm 0cm},clip]{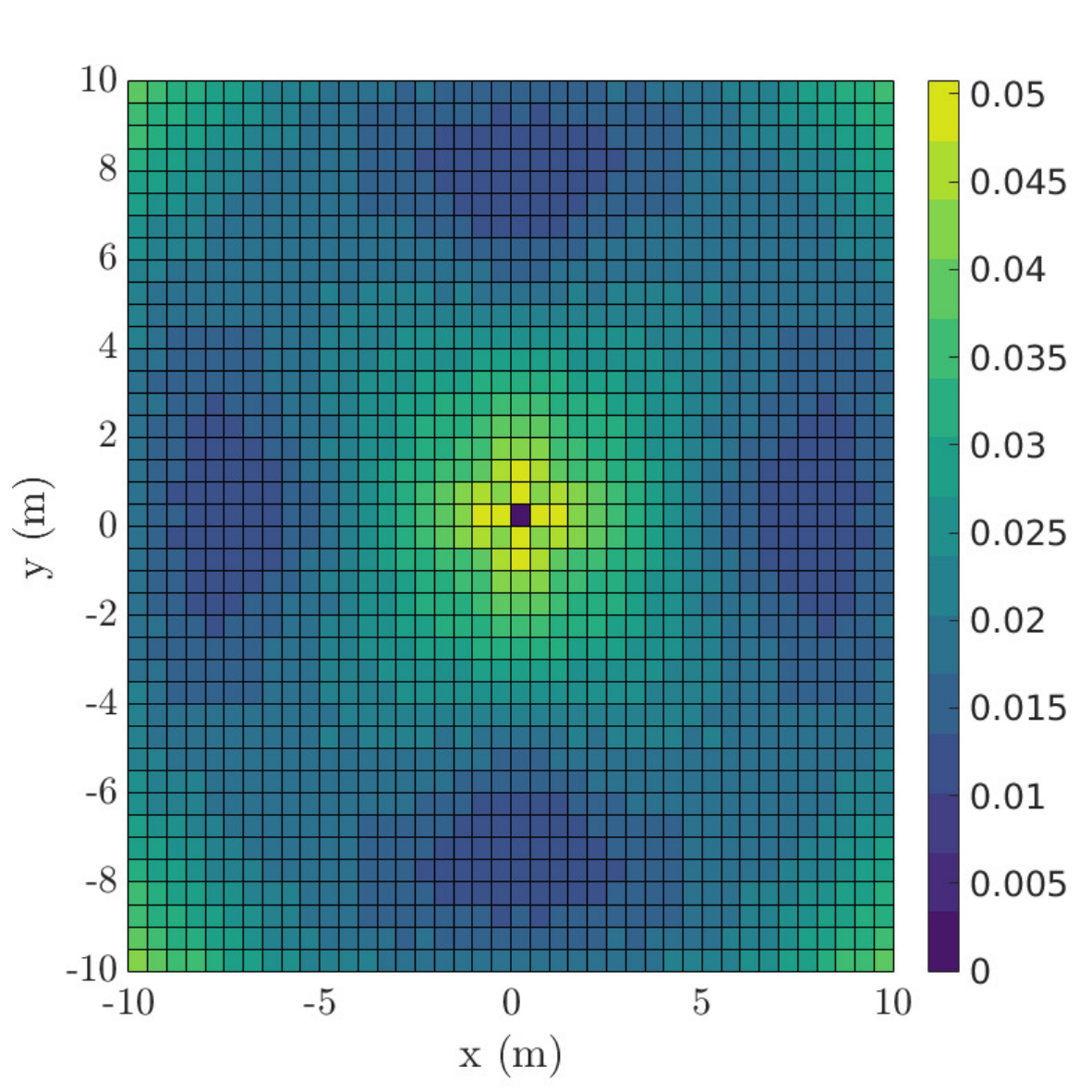}
            \caption{\hh{The area occupied by the shadow for each point within the box-constraints.}}
         \label{fig:shadow_b}
     \end{subfigure}
    \caption{\hh{Empirical 2-D visualizations of the shadow region for linearized constraints \eqref{eq:hyperplane_definition}.}}
    \label{fig:shadow}
\end{figure}%

\hh{\begin{theorem}\label{theory:redundant_scenarios}
    Suppose that the shadow is non-empty for all possible samples in the support (i.e., $S_{\b{\delta}} \neq \emptyset, \ \forall \b{\delta} \in \Delta$), then the probability that out of $S$ samples, none of the samples are redundant goes to $0$ exponentially, for $S\to\infty$.
\end{theorem}
\begin{proof}
Given $S$ samples, a new sample $\b{\delta}$ is not redundant if it falls outside of the aggregated shadow, i.e., if $\b{\delta} \notin \bigcup_{i = 1}^S S_{\b{\delta}^{(i)}}$. The associated probability is
\begin{equation}
    P_S = \mathbb{P}\left[\b{\delta} \in \Delta \ | \ \b{\delta} \notin \bigcup_{i = 1}^S S_{\b{\delta}^{(i)}}\right].
\end{equation}
Since $S_{\b{\delta}} \neq \emptyset, \ \forall \b{\delta} \in \Delta$, the aggregated shadow grows when sample $S+1$ is not redundant, i.e., $\bigcup_{i = 1}^S S_{\b{\delta}^{(i)}} \subset \bigcup_{i = 1}^{S+1} S_{\b{\delta}^{(i)}}$. Therefore, $P_{S+1} < P_{S}$ and additionally $P_i < 1$ for $i > 0$. We obtain the following result,
\begin{equation}
\lim_{S\to\infty} \mathbb{P}^S[\textrm{No redundant scenarios}] = \lim_{S\to\infty}\prod_{i = 1}^SP_i = 0,
\end{equation}
implying that the probability of sampling no redundant scenarios goes to zero. This probability is upper bounded by $(P_1)^S$, i.e., it converges at least exponentially.
\end{proof}}

\hh{Although this does not prove that the support is bounded for finite $S$, it shows that it is very likely that redundant scenarios are sampled, which are not of support.}

\end{document}